\documentclass{article}

\usepackage[preprint]{neurips_2023}
\title{Understanding Diffusion Objectives as the ELBO\\ with Simple Data Augmentation}

\author{Diederik P. Kingma \\ 
        Google DeepMind \\
        \texttt{durk@google.com}
        \And
        Ruiqi Gao\\
        Google DeepMind\\
        \texttt{ruiqig@google.com}}

\usepackage{times}
\usepackage[utf8]{inputenc} %
\usepackage[T1]{fontenc}    %
\usepackage{hyperref}       %
\usepackage[dvipsnames]{xcolor}
\hypersetup{
  colorlinks,
  linkcolor={red!50!black},
  citecolor={blue!50!black},
  urlcolor={blue!80!black}
  }
  \definecolor{orange}{HTML}{ff7f0e}
  \definecolor{blue}{HTML}{1f77b4}
  
\usepackage[framemethod=TikZ]{mdframed}
\usepackage{url}            %
\usepackage{booktabs}       %
\usepackage{amsfonts}       %
\usepackage{nicefrac}       %
\usepackage{microtype}      %
\usepackage{comment}     
\usepackage{amssymb}
\usepackage{amsmath,amsfonts,bm}
\usepackage{nicefrac}
\usepackage{textcomp}
\usepackage[capitalise]{cleveref}
\usepackage{enumitem}%
\usepackage{subcaption}
\usepackage{floatrow}
\newfloatcommand{capbtabbox}{table}[][0.48\textwidth]
\usepackage{tabularx}
\usepackage{makecell}
\usepackage{mathtools}
\usepackage[most]{tcolorbox}
\usepackage{nicefrac}
\usepackage{booktabs}
\usepackage[]{mdframed}
\usepackage{wrapfig}
\usepackage{arydshln}
\usepackage{colortbl}

\usepackage{empheq}
\usepackage[most]{tcolorbox}
\newtcbox{\mymath}[1][]{%
    nobeforeafter, math upper, tcbox raise base,
    enhanced, colframe=white!30!black,
    colback=white!30, boxrule=1pt,
    #1}

\def\eqref#1{equation~\ref{#1}}
\def\Eqref#1{Equation~\ref{#1}}

\def\1{\bm{1}}

\def\rvf{{\mathbf{f}}}

\def\rvo{{\mathbf{o}}}

\def\rvs{{\mathbf{s}}}

\def\rvv{{\mathbf{v}}}
\def\rvw{{\mathbf{w}}}
\def\rvx{{\mathbf{x}}}

\def\rvz{{\mathbf{z}}}

\def\rmF{{\mathbf{F}}}

\def\rmI{{\mathbf{I}}}

\DeclareMathAlphabet{\mathsfit}{\encodingdefault}{\sfdefault}{m}{sl}
\SetMathAlphabet{\mathsfit}{bold}{\encodingdefault}{\sfdefault}{bx}{n}

\newcommand{\E}{\mathbb{E}}

\DeclareMathOperator*{\argmax}{arg\,max}

\newcommand{\bT}{{\boldsymbol{\theta}}}

\newcommand{\pT}{p_{\bT}}

\newcommand{\snT}{\mathbf{s}_{\bT}}

\newcommand{\bfI}{\mathbf{I}}

\newcommand{\bepsilon}{{\boldsymbol{\epsilon}}}

\makeatletter
\usepackage{xspace} 
\def\@onedot{\ifx\@let@token.\else.\null\fi\xspace}
\DeclareRobustCommand\onedot{\futurelet\@let@token\@onedot}

\usepackage{amsmath}
\usepackage{amsthm}

\newcommand{\Ltfull}[1]{D_{KL}(q(\rvz_{#1}|\rvx)||p(\rvz_{#1}))}
\newcommand{\KLjoint}[1]{D_{KL}(q(\rvz_{#1})||p(\rvz_{#1}))}

\newcommand{\Lt}[1]{\mathcal{L}(#1;\rvx)}

\newcommand{\fl}{f_{\lambda}}
\newcommand{\fli}{f^{-1}_{\lambda}}

\newcommand{\tfl}{\widetilde{f}_{\lambda}}
\newcommand{\tfli}{\widetilde{f}_{\lambda}^{-1}}
\newcommand{\pl}{p}

\newcommand{\lmin}{\lambda_{\text{min}}}
\newcommand{\lmax}{\lambda_{\text{max}}}

\newtheorem{theorem}{Theorem}

\begin{document}

\maketitle

\begin{abstract}

To achieve the highest perceptual quality, state-of-the-art diffusion models are optimized with objectives that typically look very different from the maximum likelihood and the Evidence Lower Bound (ELBO) objectives. In this work, we reveal that diffusion model objectives are actually closely related to the ELBO.

Specifically, we show that all commonly used diffusion model objectives equate to a weighted integral of ELBOs over different noise levels, where the weighting depends on the specific objective used. Under the condition of monotonic weighting, the connection is even closer: the diffusion objective then equals the ELBO, combined with simple data augmentation, namely Gaussian noise perturbation. We show that this condition holds for a number of state-of-the-art diffusion models. 

In experiments, we explore new monotonic weightings and demonstrate their effectiveness, achieving state-of-the-art FID scores on the high-resolution ImageNet benchmark.
\end{abstract}

\section{Introduction}

\raggedbottom

\begin{figure*}[!b]
	\centering
     \begin{subfigure}[b]{0.32\textwidth}
         \centering
         \includegraphics[width=\textwidth]{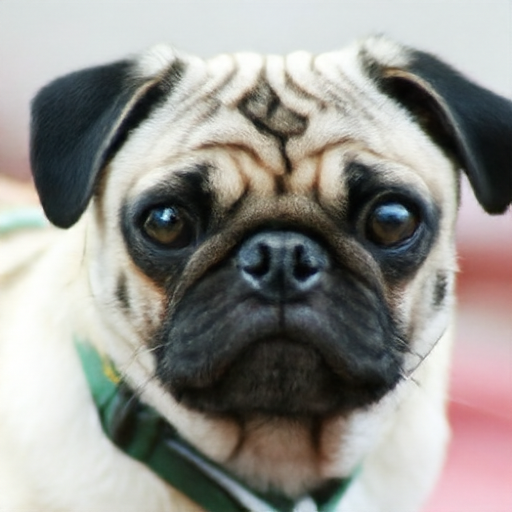}
         \caption{512 $\times$ 512}
     \end{subfigure}
     \hfill
     \begin{subfigure}[b]{0.32\textwidth}
         \centering
         \includegraphics[width=\textwidth]{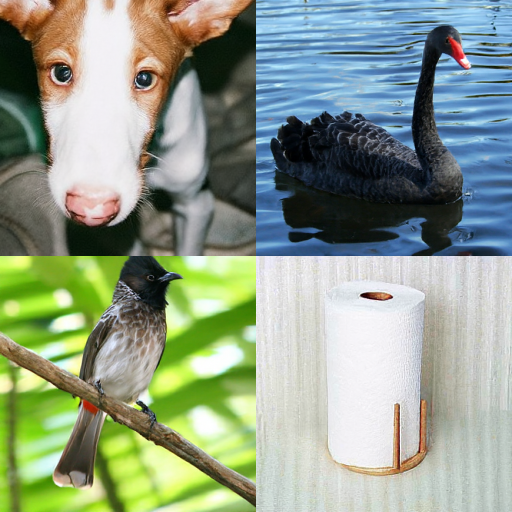}
         \caption{256 $\times$ 256}
     \end{subfigure}
     \hfill
     \begin{subfigure}[b]{0.32\textwidth}
         \centering
         \includegraphics[width=\textwidth]{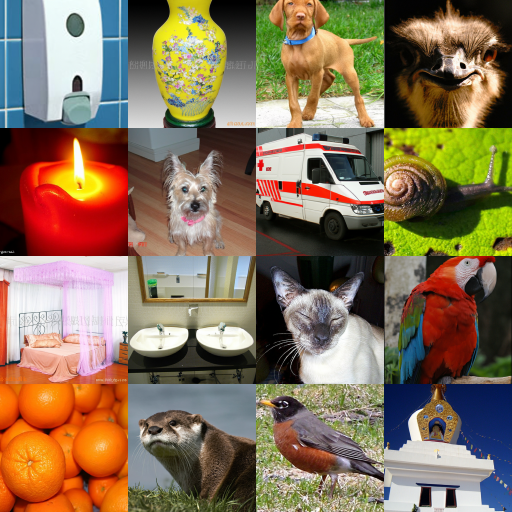}
         \caption{128 $\times$ 128}
     \end{subfigure}
        \caption{Samples generated from our diffusion models trained on the ImageNet dataset; see Section \ref{sec:experiments} for details and Appendix \ref{sec:samples} for more samples.}
        \label{fig:i128_sample}
\end{figure*} 

Diffusion-based generative models, or diffusion models in short, were first introduced by \cite{sohl2015deep}. After years of relative obscurity, this class of models suddenly rose to prominence with the work of \cite{song2019generative} and \cite{ho2020denoising} who demonstrated that, with further refinements in model architectures and objective functions, diffusion models can perform state-of-the-art image generation.

The diffusion model framework has since been successfully applied to text-to-image generation \citep{rombach-cvpr-2022,nichol-glide,ramesh-dalle2,sahariac-imagen,yu-parti-2022,nichol2021improved,ho2021cascaded,dhariwal2021diffusion,CogView}, image-to-image generation \citep{saharia2021image,sahariac-palette,whang-cvpr-2022}, 3D synthesis \citep{poole2022dreamfusion,watson20223dim}, text-to-speech \citep{chen-iclr-2021,kong-arxiv-2020,chen-interspeech-2021}, and density estimation \citep{kingma2021variational,song2021maximum}.

Diffusion models can be interpreted as a special case of deep variational autoencoders (VAEs) \citep{kingma2013auto,rezende2014stochastic} with a particular choice of inference model and generative model. Just like VAEs, the original diffusion models~\citep{sohl2015deep} were optimized by maximizing the variational lower bound of the log-likelihood of the data, also called the evidence lower bound, or ELBO for short. It was shown by \emph{Variational Diffusion Models} (VDM) \citep{kingma2021variational} and \citep{song2021maximum} how to optimize \emph{continuous-time} diffusion models with the ELBO objective, achieving state-of-the-art likelihoods on image density estimation benchmarks. 

However, the best results in terms of sample quality metrics such as FID scores were achieved with other objectives, for example a denoising score matching objective \citep{song2019generative} or a simple noise-prediction objective \citep{ho2020denoising}. These now-popular objective functions look, on the face of it, very different from the traditionally popular maximum likelihood and ELBO objectives. 
Through the analysis in this paper, we reveal that all training objective used in state-of-the-art diffusion models are actually closely related to the ELBO objective.

This paper is structured as follows:
\begin{itemize}
\item In Section \ref{sec:model_family} we introduce the broad diffusion model family under consideration.
\item In Section \ref{sec:diffobj}, we show how the various diffusion model objectives in the literature can be understood as special cases of a \emph{weighted loss} \citep{kingma2021variational,song2021maximum}, with different choices of weighting. The weighting function specifies the weight per noise level. In Section \ref{sec:invariance} we show that during training, the noise schedule acts as a importance sampling distribution for estimating the loss, and is thus important for efficient optimization. Based on this insight we propose a simple adaptive noise schedule.
\item In Section \ref{sec:mainresult}, we present our main result: that if the weighting function is a monotonic function of time, then the weighted loss corresponds to maximizing the ELBO with data augmentation, namely Gaussian noise perturbation. This holds for, for example, the $\rvv$-prediction loss of ~\citep{salimans2022progressive} and flow matching with the optimal transport path~\citep{lipman2022flow}.
\item In Section \ref{sec:experiments} we perform experiments with various new monotonic weights on the ImageNet dataset, and find that our proposed monotonic weighting produces models with sample quality that are competitive with the best published results, achieving state-of-art FID and IS scores on high resolution ImageNet generation.
\end{itemize}

\subsection{Related work}
\label{sec:related_work}

The main sections reference much of the related work. Earlier work \citep{kingma2021variational,song2021maximum,huang2021variational,vahdat2021score}, including Variational Diffusion Models \citep{kingma2021variational}, showed how to optimize continous-time diffusion models towards the ELBO objective. We generalize these earlier results by showing that any diffusion objective that corresponds with monotonic weighting corresponds to the ELBO, combined with a form of DistAug \citep{child2019generating}. DistAug is a method of training data distribution augmentation for generative models where the model is conditioned on the data augmentation parameter at training time, and conditioned on 'no augmentation' at inference time. The type of data augmentation under consideration in this paper, namely additive Gaussian noise, is also a form of data distribution smoothing, which has been shown to improve sample quality in autoregressive models by \cite{meng2021improved}.

\cite{kingma2021variational} showed how the ELBO is invariant to the choice of noise schedule, except for the endpoints. We generalize this result by showing that the invariance holds for any weighting function.

\section{Model}
\label{sec:model_family}

Suppose we have a dataset of datapoints drawn from $q(\rvx)$. We wish to learn a generative model $\pT(\rvx)$ that approximates $q(\rvx)$. We'll use shorthand notation $p := \pT$.

The observed variable $\rvx$ might be the output of a pre-trained encoder, as in \emph{latent diffusion models} ~\citep{vahdat2021score,rombach2022high}, on which the popular \emph{Stable Diffusion} model is based. Our theoretical analysis also applies to this type of model.

In addition to the observed variable $\rvx$, we have a series of latent variables $\rvz_t$ for timesteps $t \in [0,1]$: $\rvz_{0, ..., 1} := \rvz_0, ..., \rvz_1$. The model consists of two parts: a \emph{forward process} forming a conditional joint distribution $q(\rvz_{0, ..., 1}|\rvx)$, and a \emph{generative model} forming a joint distribution $p(\rvz_{0, ..., 1})$.

\subsection{Forward process and noise schedule}
\label{sec:noise_schedule}

\begin{figure}[t]
    \centering
    \includegraphics[width=1.0\textwidth]{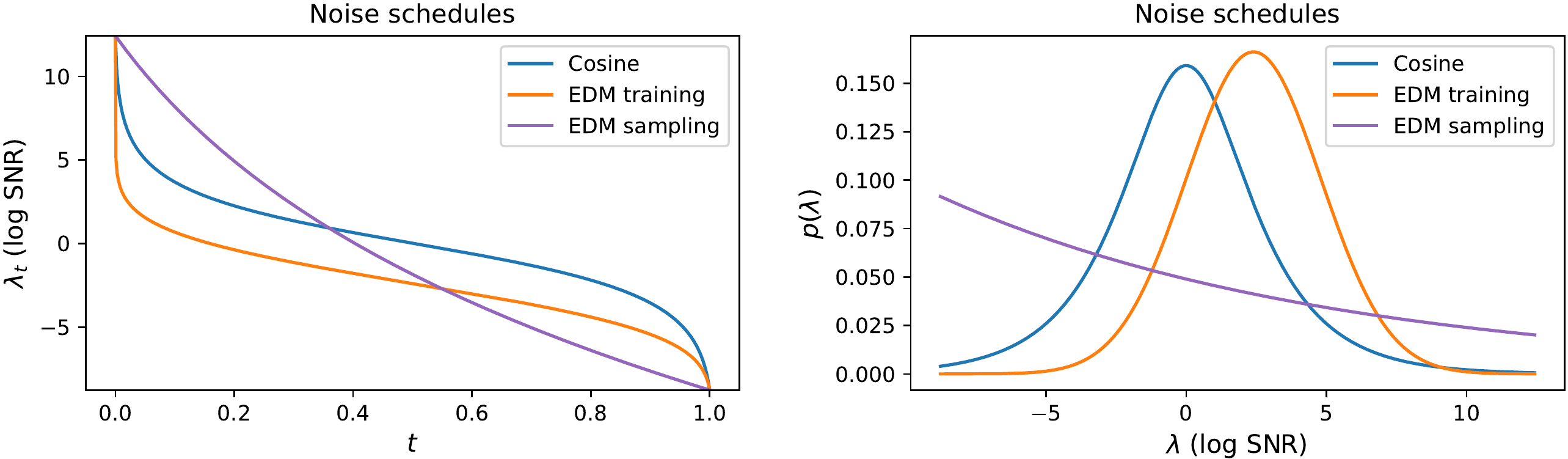}
    \caption{\textbf{Left:} Noise schedules used in our experiments: cosine \citep{nichol2021improved} and EDM \citep{karras2022elucidating} training and sampling schedules. \textbf{Right:} The same noise schedules, expressed as probability densities $p(\lambda) = -dt/d\lambda$. See Section \ref{sec:noise_schedule} and Appendix \ref{sec:schedules} for details.}
    \label{fig:noise_schedules}
\end{figure}

The forward process is a Gaussian diffusion process, giving rise to a conditional distribution $q(\rvz_{0, ..., 1}|\rvx)$; see Appendix \ref{sec:sdes} for details. For every $t \in [0,1]$, the marginal distribution $q(\rvz_t|\rvx)$ is given by:
\begin{align}
\rvz_t =  \alpha_\lambda \rvx + \sigma_\lambda \bepsilon \;\;\text{where}\;\; \bepsilon \sim \mathcal{N}(0,\bfI)
\end{align}
In case of the often-used \emph{variance preserving} (VP) forward process, $\alpha^2_\lambda = \text{sigmoid}(\lambda_t)$ and $\sigma^2_\lambda = \text{sigmoid}(-\lambda_t)$, but other choices are possible; our results are agnostic to this choice. The \emph{log signal-to-noise ratio} (log-SNR) for timestep $t$ is given by $\lambda = \log(\alpha^2_\lambda / \sigma^2_\lambda)$. 

The \emph{noise schedule} is a strictly monotonically decreasing function $\fl$ that maps from the time variable $t \in [0,1]$ to the corresponding log-SNR $\lambda$: $\lambda = \fl(t)$. We sometimes denote the log-SNR as $\lambda_t$ to emphasize that it is a function of $t$. The endpoints of the noise schedule are given by $\lmax := \fl(0)$ and $\lmin := \fl(1)$. See Figure \ref{fig:noise_schedules} for a visualization of commonly used noise schedules in the literature, and Appendix \ref{sec:schedules} for more details. 

Due to its monotonicity, $\fl$ is invertible: $t = \fli(\lambda)$. Given this bijection, we can do a change of variables: a function of the value $t$ can be equivalently written as a function of the corresponding value $\lambda$, and vice versa, which we'll make use of in this work.

During model training, we sample time $t$ uniformly: $t \sim \mathcal{U}(0,1)$, then compute $\lambda = \fl(t)$. This results in a distribution over noise levels $p(\lambda) = -dt/d\lambda = -1/\fl'(t)$ (see Section \ref{sec:schedules}), which we also plot in Figure \ref{fig:noise_schedules}.

Sometimes it is best to use a different noise schedule for sampling from the model than for training. During sampling, the density $p(\lambda)$ gives the relative amount of time the sampler spends at different noise levels.

\subsection{Generative model}

The data $\rvx \sim \mathcal{D}$, with density $q(\rvx)$, plus the forward model defines a joint distribution $q(\rvz_0, ..., \rvz_1) = \int q(\rvz_0, ..., \rvz_1|\rvx) q(\rvx) d\rvx$, with marginals $q_t(\rvz) := q(\rvz_t)$. The generative model defines a corresponding joint distribution over latent variables: $p(\rvz_0, ..., \rvz_1)$. 

For large enough $\lmax$, $\rvz_0$ is almost identical to $\rvx$, so learning a model $p(\rvz_0)$ is practically equivalent to learning a model $p(\rvx)$. For small enough $\lmin$, $\rvz_1$ holds almost no information about $\rvx$, such that there exists a distribution $p(\rvz_1)$ satisfying $\Ltfull{1} \approx 0$. Usually we can use $p(\rvz_1) = \mathcal{N}(0,\bfI)$.

Let $\snT(\rvz; \lambda)$ denote a \emph{score model}, which is a neural network that we let approximate $\nabla_{\rvz} \log q_t(\rvz)$ through methods introduced in the next sections. If $\snT(\rvz;\lambda) = \nabla_{\rvz} \log q_t(\rvz)$, then the forward process can be exactly reversed; see Appendix \ref{sec:sampling}.

If $\KLjoint{1} \approx 0$ and $\snT(\rvz; \lambda) \approx \nabla_{\rvz} \log q_t(\rvz)$, then we have a good generative model in the sense that $\KLjoint{0,...,1} \approx 0$, which implies that $\KLjoint{0} \approx 0$ which achieves our goal. So, our generative modeling task is reduced to learning a score network $\snT(\rvz; \lambda)$ that approximates $\nabla_{\rvz} \log q_t(\rvz)$.

Sampling from the generative model can be performed by sampling $\rvz_1 \sim p(\rvz_1)$, then (approximately) solving the reverse SDE using the estimated $\snT(\rvz; \lambda)$. Recent diffusion models have used increasingly sophisticated procedures for approximating the reverse SDE; see Appendix \ref{sec:sampling}. In experiments we use the DDPM sampler from~\citet{ho2020denoising} and the stochastic sampler with Heun's second order method proposed by~\cite{karras2022elucidating}.

\section{Diffusion Model Objectives}
\label{sec:diffobj}

\definecolor{burntorange}{rgb}{0.8, 0.33, 0.0}
\newcommand{\colorA}[1]{\textcolor{burntorange}{#1}}
\newcommand{\colorB}[1]{\textcolor{blue}{#1}}
\newcommand{\colorC}[1]{\textcolor{red}{#1}}
\newcommand{\colorD}[1]{\textcolor{ForestGreen}{#1}}
\newcommand{\colorE}[1]{\textcolor{RoyalPurple}{#1}}

\paragraph{Denoising score matching.}
Above, we saw that we need to learn a score network $\snT(\rvz; \lambda_t)$ that approximates $\nabla_{\rvz} \log q_t(\rvz)$, for all noise levels $\lambda_t$. It was shown by ~\citep{vincent2011connection,song2019generative} that this can be achieved by minimizing a denoising score matching objective over all noise scales and all datapoints $\rvx \sim \mathcal{D}$:
\begin{align*}
\mathcal{L}_{\text{DSM}}(\rvx) 
&= 
\E_{t \sim \mathcal{U}(0,1), \bepsilon \sim \mathcal{N}(0,\mathbf{I})} \left[
\colorA{\tilde{w}(t)} \cdot || 
\colorC{\rvs_{\bT}(\rvz_t, \lambda_t)} - \colorE{\nabla_{\rvz_t} \log q(\rvz_t | \rvx)}
||_2^2
\right]
\end{align*}
where $\rvz_t = \alpha_\lambda \rvx + \sigma_\lambda \bepsilon$.

\paragraph{The $\epsilon$-prediction objective.}
Most contemporary diffusion models are optimized towards a noise-prediction loss introduced by \citep{ho2020denoising}. In this case, the score network is typically parameterized through a noise-prediction ($\bepsilon$-prediction) model: $\snT(\rvz; \lambda) = - \hat{\bepsilon}_{\bT}(\rvz;\lambda) / \sigma_\lambda$. (Other options, such as $\rvx$-prediction $\rvv$-prediction, and EDM parameterizations, are explained in Appendix \ref{sec:param}.) The noise-prediction loss is:
\begin{align}
\mathcal{L}_{\bepsilon}(\rvx) 
&=
\frac{1}{2} 
\E_{t \sim \mathcal{U}(0,1), \bepsilon \sim \mathcal{N}(0,\mathbf{I})} 
\left[
|| 
\colorC{\hat{\bepsilon}_{\bT}(\rvz_t;\lambda_t)} - \colorE{\bepsilon}
||_2^2
\right]\label{eq:l_simple}
\end{align}
Since $||\rvs_{\bT}(\rvz_t, \lambda_t) - \nabla_{\rvx_t} \log q(\rvz_t | \rvx)
||_2^2 = \sigma_\lambda^{-2} || 
\hat{\bepsilon}_{\bT}(\rvz_t;\lambda_t) - \bepsilon
||_2^2$, this is simply a version of the denoising score matching objective above, but where $\tilde{w}(t) = \sigma_t^2$. \cite{ho2020denoising} showed that this noise-prediction objective can result in high-quality samples. \cite{dhariwal2021diffusion} later improved upon these results by switching from a `linear' to a `cosine' noise schedule $\lambda_t$ (see Figure \ref{fig:noise_schedules}). This noise-prediction loss with the cosine schedule is currently broadly used.

\paragraph{The ELBO objective.}

It was shown by \citep{kingma2021variational,song2021maximum} that the evidence lower bound (ELBO) of continuous-time diffusion models simplifies to:
\begin{align}
-\text{ELBO}(\rvx) &=
\frac{1}{2} 
\E_{t \sim \mathcal{U}(0,1), \bepsilon \sim \mathcal{N}(0,\mathbf{I})} 
\left[
-\colorA{\frac{d\lambda}{dt}} \cdot || 
\colorC{\hat{\bepsilon}_{\bT}(\rvz_t;\lambda_t)} - \colorE{\bepsilon}
||_2^2
\right]
+ c
\label{eq:l_elbo}
\end{align}
where $c$ is constant w.r.t. the score network parameters.

\subsection{The weighted loss}
\label{sec:weighted_loss}

\begin{table}[t]
\caption{Diffusion model objectives in the literature are special cases of the weighted loss with a weighting function $w(\lambda)$ given in this table. See Section \ref{sec:weighted_loss} and Appendix \ref{sec:loss_functions} for more details and derivations. Most existing weighting functions are non-monotonic, except for the ELBO objective and the $\rvv$-prediction objective with `cosine' schedule.
}
\label{table:w}
\scriptsize
\begin{center}
{\renewcommand{\arraystretch}{1.2}%
\begin{tabular}{ lll }
\toprule
\textbf{Loss function} & \textbf{Implied weighting} $w(\lambda)$
& \textbf{Monotonic?}
\\ 
\midrule
ELBO \citep{kingma2021variational,song2021maximum} & 1 & \checkmark\\
IDDPM ($\bepsilon$-prediction with 'cosine' schedule) \citep{nichol2021improved}& $\text{sech}(\lambda/2)$\\ 
EDM \citep{karras2022elucidating} (Appendix \ref{sec:edm})& $\mathcal{N}(\lambda; 2.4, 2.4^2) \cdot (e^{-\lambda} + 0.5^2)$\\ 
$\rvv$-prediction with `cosine'  schedule \citep{salimans2022progressive} (Appendix \ref{sec:vparam}) 
& $e^{-\lambda/2}$ & \checkmark\\ 
Flow Matching with OT path (FM-OT) \citep{lipman2022flow} (Appendix \ref{sec:fm})
& $e^{-\lambda/2}$ & \checkmark
\\
InDI \citep{delbracio2023inversion} (Appendix \ref{sec:indi}) & $e^{-\lambda} \text{sech}^2(\lambda/4)$ & \checkmark
\\
P2 weighting with `cosine' schedule~\citep{choi2022perception} (Appendix \ref{sec:p2-weighting}) & $\text{sech}(\lambda/2) / (1 + e^\lambda)^\gamma$, $\gamma = 0.5$ or $1$ & \\
Min-SNR-$\gamma$~\citep{hang2023efficient} (Appendix \ref{sec:min-snr}) & $\text{sech}(\lambda/2) \cdot \min(1, \gamma e^{-\lambda})$ \\
\bottomrule
\end{tabular}
}
\end{center}
\end{table}

\begin{figure}[t]
    \centering
    \includegraphics[width=1.0\textwidth]{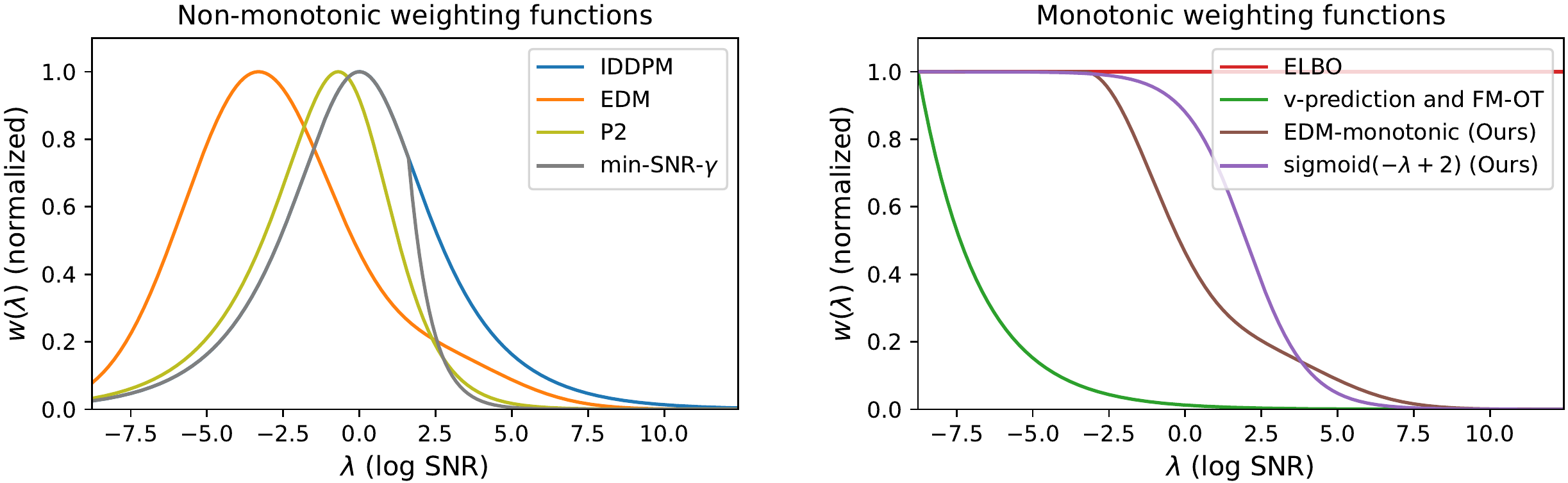}
    \caption{Diffusion model objectives in the literature are special cases of the weighted loss with non-monotonic (left) or monotonic (right) weighting functions. Each weighting function is scaled such that the maximum is 1 over the plotted range. See Table \ref{table:w} and Appendix \ref{sec:loss_functions}.}
    \label{fig:weighting_functions}
\end{figure}

The objective functions used in practice, including the ones above, can be viewed as special cases of a \emph{weighted loss} introduced by \cite{kingma2021variational}%
\footnote{More specifically, ~\cite{kingma2021variational} expressed the weighted diffusion loss in terms of $\rvx$-prediction, which is equivalent to the expression above due to the relationship $\int ||\bepsilon - \hat{\bepsilon}_{\bT}(\rvz_t;\lambda)||_2^2 \,d\lambda = \int ||\rvx - \hat{\rvx}_{\bT}(\rvz_t;\lambda)||_2^2 e^\lambda d\lambda = \int ||\rvx - \hat{\rvx}_{\bT}(\rvz_t;\lambda)||_2^2 \,de^\lambda$, where $e^\lambda$ equals the signal-to-noise ratio (SNR).}%
 with a particular choice of weighting function $w(\lambda_t)$:
\begin{empheq}[box={\mymath}]{align}
\mathcal{L}_w(\rvx) 
&=
\frac{1}{2} 
\E_{t \sim \mathcal{U}(0,1), \bepsilon \sim \mathcal{N}(0,\mathbf{I})} 
\left[
\colorD{w(\lambda_t)}
\cdot
-\colorA{\frac{d\lambda}{dt}}
\cdot || 
\colorC{\hat{\bepsilon}_{\bT}(\rvz_t;\lambda_t)} - \colorE{\bepsilon}
||_2^2
\right]
\label{eq:lw_new}
\end{empheq}
See Appendix \ref{sec:loss_functions} for a derivation of the implied weighting functions for all popular diffusion losses. Results are compiled in Table \ref{table:w}, and visualized in Figure \ref{fig:weighting_functions}. 

The ELBO objective (\Eqref{eq:l_elbo}) corresponds to uniform weighting, i.e. $w(\lambda_t)=1$. 

The noise-prediction objective (\Eqref{eq:l_simple}) corresponds to $\mathcal{L}_w(\rvx)$ with $w(\lambda_t) = -1/(d\lambda/dt)$. This is more compactly expressed as $w(\lambda_t) = p(\lambda_t)$, i.e., the PDF of the implied distribution over noise levels $\lambda$ during training. Typically, the noise-prediction objective is used with the cosine schedule $\lambda_t$, which implies $w(\lambda_t) \propto \text{sech}(\lambda_t/2)$. See Section \ref{sec:schedules} for the expression of $\pl(\lambda_t)$ for various noise schedules.

\subsection{Invariance of the weighted loss to the noise schedule $\lambda_t$}
\label{sec:invariance}

In \cite{kingma2021variational}, it was shown that the ELBO objective (\Eqref{eq:l_elbo}) is invariant to the choice of noise schedule, except for its endpoints $\lmin$ and $\lmax$. This result can be extended to the much more general weighted diffusion loss of \Eqref{eq:lw_new}, since with a change of variables from $t$ to $\lambda$, it can be rewritten to:
\begin{align}
\mathcal{L}_w(\rvx) 
=
\frac{1}{2} 
\int_{\colorA{\lmin}}^{\colorA{\lmax}}
\colorD{w(\lambda)}
\E_{\bepsilon \sim \mathcal{N}(0,\mathbf{I})} 
\left[
|| 
\colorC{\hat{\bepsilon}_{\bT}(\rvz_\lambda;\lambda)} - \colorE{\bepsilon}
||_2^2
\right]
\,\colorA{d\lambda}
\label{eq:lw_integral}
\end{align}
Note that this integral does not depend on the noise schedule $\fl$ (the mapping between $t$ and $\lambda$), except for its endpoints $\lmax$ and $\lmin$. This is important, since it means that the shape of the function $\fl$ between $\lmin$ and $\lmax$ does not affect the loss; only the weighting function $w(\lambda)$ does. Given a chosen weighting function $w(\lambda)$, the loss is invariant to the noise schedule $\lambda_t$ between $t=0$ and $t=1$. The only real difference between diffusion objectives is their difference in weighting $w(\lambda)$.

This invariance does \emph{not} hold for the Monte Carlo estimator of the loss that we use in training, based on random samples $t \sim \mathcal{U}(0,1), \bepsilon \sim \mathcal{N}(0,\mathbf{I})$. The noise schedule still affects the \emph{variance} of this Monte Carlo estimator and its gradients; therefore, the noise schedule affects the efficiency of optimization. 
More specifically, we'll show that the noise schedule acts as an importance sampling distribution for estimating the loss integral of \Eqref{eq:lw_integral}. 
Note that $p(\lambda) = -1/(d\lambda/dt)$. We can therefore rewrite the weighted loss as the following, which clarifies the role of $p(\lambda)$ as an importance sampling distribution:
\begin{empheq}[box={\mymath}]{align}
\mathcal{L}_w(\rvx) = 
\frac{1}{2} 
\E_{\bepsilon \sim \mathcal{N}(0,\mathbf{I}), \lambda \sim p(\lambda)} 
\left[
\frac{\colorD{w(\lambda)}}
{\colorA{p(\lambda)}}
|| 
\colorC{\hat{\bepsilon}_{\bT}(\rvz_\lambda;\lambda)} - \colorE{\bepsilon}
||_2^2
\right]
\end{empheq}
Using this insight, we implemented an adaptive noise schedule, detailed in Appendix \ref{sec:adaptive}. We find that by lowering the variance of the loss estimator, this often significantly speeds op optimization.

\section{The weighted loss as the ELBO with data augmentation}
\label{sec:mainresult}

We prove the following main result:

\begin{mdframed}
\begin{theorem}
If the weighting $w(\lambda_t)$ is monotonic, then the weighted diffusion objective of \Eqref{eq:lw_new} is equivalent to the ELBO with data augmentation (additive noise).
\label{theorem1}
\end{theorem}
\end{mdframed}
With monotonic $w(\lambda_t)$ we mean that $w$ is a monotonically increasing function of $t$, and therefore a monotonically decreasing function of $\lambda$.

We'll use shorthand notation $\Lt{t}$ for the KL divergence between the joint distributions of the forward process $q(\rvz_{t,...1}|\rvx)$ and the reverse model $p(\rvz_{t,...1})$, for the subset of timesteps from $t$ to $1$:
\begin{empheq}[box={\mymath}]{equation}
\colorB{\Lt{t}} := \colorB{\Ltfull{t,...,1}}
\end{empheq}
In Appendix \ref{app:time_derivative_of_dlt}, we prove that\footnote{Interestingly, this reveals a relationship with the Fisher divergence: $\frac{d}{d\lambda} \, \Ltfull{t,...,1}
= \frac{1}{2} \sigma_\lambda^2 D_{F}(q(\rvz_t|\rvx)||p(\rvz_t))$.
See Appendix \ref{sec:fisher} for a derivation and a comparison with a similar result by~\cite{lyu2012interpretation}.}:
\begin{align}
\colorB{
\frac{d}{dt} \Lt{t}}
=
\colorC{
\frac{1}{2} 
\frac{d\lambda}{dt}
\E_{\bepsilon \sim \mathcal{N}(0, \mathbf{I})}\left[
||\bepsilon - \hat{\bepsilon}_{\bT}(\rvz_\lambda;\lambda)||_2^2
\right]
}
\label{eq:time_deriv}
\end{align}
As shown in Appendix \ref{app:time_derivative_of_dlt}, this allows us to rewrite the weighted loss of \Eqref{eq:lw_new} as simply:
\begin{align}
\mathcal{L}_w(\rvx) 
&=
-
\int_0^1
\colorB{
\frac{d}{dt} \Lt{t}}
\, \colorD{w(\lambda_t)}
\;dt
\end{align}
In Appendix \ref{sec:intbyparts}, we prove that using integration by parts, the weighted loss can then be rewritten as:
\begin{align}
\mathcal{L}_w(\rvx) 
=&
\int_0^1
\colorD{\frac{d}{dt} w(\lambda_t)}
\,\colorB{\Lt{t}}
\;dt
+ \colorD{w(\lmax)} \,\colorB{\Lt{0}}
+ \text{constant}
\label{eq:mainresult}
\end{align}
Now, assume that $w(\lambda_t)$ is a monotonically increasing function of $t \in [0,1]$. Also, without loss of generality, assume that $w(\lambda_t)$ is normalized such that $w(\lambda_1)=1$. We can then further simplify to an expected KL divergence:
\begin{empheq}[box={\mymath[]}]{equation}
\mathcal{L}_w(\rvx)
=
\E_{\colorD{p_w(t)}} 
\left[
\colorB{
\Lt{t}}
\right]
+ \text{constant}
\label{eq:mainresult2}
\end{empheq}
where $p_w(t)$ is a probability distribution determined by the weighting function, namely $p_w(t) := (d/dt\;w(\lambda_t))$, with support on $t \in [0, 1]$. The probability distribution $p_w(t)$ has Dirac delta peak of typically very small mass $w(\lmax)$ at $t=0$. 

Note that:
\begin{align}
\Lt{t} &= \Ltfull{t,...,1}\\
&\geq D_{KL}(q(\rvz_t|\rvx)||p(\rvz_t)) = - \E_{q(\rvz_t|\rvx)}[ \log p(\rvz_t)] + \text{constant.}
\end{align}
More specifically, $\Lt{t}$ equals the expected negative ELBO of noise-perturbed data, plus a constant; see Section \ref{sec:elbo} for a detailed derivation.

This concludes our proof of Theorem \ref{theorem1}.
\renewcommand{\qedsymbol}{$\blacksquare$}
\qed

This result provides us the new insight that any of the objectives with (implied) monotonic weighting, as listed in Table \ref{table:w}, can be understood as equivalent to the ELBO with simple data augmentation, namely additive noise. Specifically, this is a form of Distribution Augmentation (DistAug), where the model is conditioned on the augmentation indicator during training, and conditioned on 'no augmentation' during sampling.

Monotonicity of $w(\lambda)$ holds for a number of models with state-of-the-art perceptual quality, including VoiceBox for speech generation~\citep{le2023voicebox}, and Simple Diffusion for image generation~\citep{hoogeboom2023simple}.

\section{Experiments}
\label{sec:experiments}

Inspired by the theoretical results in Section \ref{sec:mainresult}, in this section we proposed several monotonic weighting functions, and conducted experiments to test the effectiveness of the monotonic weighting functions compared to baseline (non-monotonic) weighting functions. In addition, we test the adaptive noise schedule (Section \ref{sec:invariance}). For brevity we had to pick a name for our models. 
Since we build on earlier results on Variational Diffusion Models (VDM)~\citep{kingma2021variational}, and our objective is shown to be equivalent to the VDM objective combined with data augmentation, we name our models \emph{VDM++}.

\subsection{ImageNet 64x64}
\label{sec:imagenet64}

\begin{table}[t]
\scriptsize
\centering
\caption{ImageNet 64x64 results. See Section \ref{sec:imagenet64}.}
\label{table:imagenet64}
{\renewcommand{\arraystretch}{1.2}%
\setlength\tabcolsep{2.5pt}
\begin{tabular}{ lllccccc }
\toprule
&
& 
&
& \multicolumn{2}{c}{DDPM sampler} 
& \multicolumn{2}{c}{EDM sampler} 
\\
\textbf{Model parameterization}
& \textbf{Training noise schedule}
& \textbf{Weighting function} 
& \textbf{Monotonic?}
& \textbf{FID} $\downarrow$
& \textbf{IS} $\uparrow$
& \textbf{FID} $\downarrow$
& \textbf{IS} $\uparrow$
\\ 
\midrule
$\bepsilon$-prediction %
& Cosine
& $\text{sech}(\lambda/2)$ (Baseline)
&
& 1.85
& 54.1 $\pm$ 0.79
& 1.55
& 59.2 $\pm$ 0.78
\\
"%
& Cosine
& $\text{sigmoid}(-\lambda + 1)$
& \checkmark
& 1.75
& 55.3 $\pm$ 1.23
& 
&
\\
"
& Cosine
& $\text{sigmoid}(-\lambda + 2)$
& \checkmark
& \textbf{1.68}
& \textbf{56.8} $\pm$ 0.85
& 1.46
& 60.4 $\pm$ 0.86
\\
"
& Cosine
& $\text{sigmoid}(-\lambda + 3)$
& \checkmark
& 1.73
& 56.1 $\pm$ 1.36
&
& 
\\
"
& Cosine
& $\text{sigmoid}(-\lambda + 4)$
& \checkmark
& 1.80
& 55.1 $\pm$ 1.65
& 
&
\\
"
& Cosine
& $\text{sigmoid}(-\lambda + 5)$
& \checkmark
& 1.94
& 53.5 $\pm$ 1.12
& 
&
\\
"
& Adaptive
& $\text{sigmoid}(-\lambda + 2)$
& \checkmark
& 1.70
& 54.8 $\pm$ 1.20
& \textbf{1.44}
& 60.6 $\pm$ 1.44
\\
"
& Adaptive
& EDM-monotonic
& \checkmark
& \textbf{1.67}
& \textbf{56.8} $\pm$ 0.90
& \textbf{1.44}
& \textbf{61.1} $\pm$ 1.80
\\
\midrule
EDM \citep{karras2022elucidating}
& EDM (training)
& EDM (Baseline)
& 
& & & 1.36 &
\\
EDM (our reproduction)
& EDM (training)
& EDM (Baseline)
& 
& & & 1.45
& 60.7 $\pm$ 1.19
\\
"
& Adaptive
& EDM
& 
& & & \textbf{1.43}
& 63.2 $\pm$ 1.76 
\\
"
& Adaptive
& $\text{sigmoid}(-\lambda + 2)$
& \checkmark
& & & 1.55
& \textbf{63.7} $\pm$ 1.14
\\
"
& Adaptive
& EDM-monotonic
& \checkmark
& & & \textbf{1.43}
& \textbf{63.7} $\pm$ 1.48 
\\

\midrule
$\rvv$-prediction %
& Adaptive
& $\exp(-\lambda/2)$ (Baseline)
& \checkmark
& & & 1.62
& 58.0 $\pm$ 1.56
\\
"%
& Adaptive
& $\text{sigmoid}(-\lambda + 2)$
& \checkmark
& & & 1.51
& 64.4 $\pm$ 1.28
\\
"
& Adaptive
& EDM-monotonic
& \checkmark
& & & \textbf{1.45}
& \textbf{64.6} $\pm$ 1.35
\\
\bottomrule
\end{tabular}
}
\end{table}

All experiments on ImageNet 64x64 were done with the U-Net diffusion model architecture from \citep{nichol2021improved}. Table \ref{table:imagenet64} summarizes the FID~\citep{heusel2017gans} and Inception scores~\citep{salimans2016improved} across different settings.

We started with the $\bepsilon$-prediction model, and the $\bepsilon$-prediction loss with cosine noise schedule, which is the most popular setting in the literature and therefore serves as a reasonable baseline. This corresponds to a non-monotonic $\text{sech}(\lambda/2)$ weighting. We replaced this weighting with a monotonic weighting, specifically a sigmoidal weighting of the form $\text{sigmoid}(-\lambda + k)$, with $k \in \{1,2,3,4,5\}$. We observed an improvement in FID and Inception scores in all experiments except for $k = 5$, with the best FID and Inception scores at $k=2$. See Figure \ref{fig:weighting_functions} for visualization of the sigmoidal weighting with $k=2$. These initial experiments were performed with the DDPM sampler from \cite{ho2020denoising}, but scores improved further by switching to the EDM sampler from \cite{karras2022elucidating}, so we used the EDM sampler for all subsequent experiments. For fair comparison, we fixed the noise schedule used in sampling to the cosine schedule for the DDPM sampler, and the EDM (sampling) schedule for the EDM sampler (see Table \ref{table:noise_schedules} for the formulations). Then, we changed the training noise schedule from the cosine schedule to our adaptive schedule (Appendix \ref{sec:adaptive_main}), resulting in similar scores, but with faster training. We also proposed another monotonic weighting `EDM-monotonic', introduced in the next paragraph, that was inspired by \citet{karras2022elucidating}. It performs slightly better than the best sigmoidal weighting for the $\bepsilon$-prediction model. 

Next, we trained models using the EDM parameterization from \cite{karras2022elucidating}, which reported the best FIDs in the literature for ImageNet 64x64. Our re-implementation could not exactly reproduce their reported FID number (1.36), but comes close (1.45). We observed a slight improvement in scores by switching from the EDM training noise schedule to the adaptive noise schedule. We found that replacing the non-monotonic EDM weighting with the best-performing (monotonic) sigmoidal weighting from the previous paragraph, resulted in a slightly worse FID and a slightly better inception score. Inspired by the EDM weighting function from~\citet{karras2022elucidating} (non-monotonic, Table \ref{table:w}), we trained a model using a weighting function indicated with `EDM-monotonic', which is identical to the EDM weighting, except that it is made monotonic by letting $w(\lambda) = \max_{\lambda} \tilde{w}(\lambda)$ for $\lambda < \argmax_{\lambda} \tilde{w}(\lambda)$, where $\tilde{w}(\lambda)$ indicates the original EDM weighting function. Hence, `EDM-monotonic' is identical to the EDM weighting to the right of its peak, but remains as a constant to left of its peak. This monotonic weighting function leads to scores on par with the original EDM weighting. Interestingly, we didn't get significantly better scores with the EDM parameterization than with the $\bepsilon$-prediction parameterization in the previous paragraph when combined with monotonic weightings.

We also trained a $\rvv$-prediction model with original weighting, sigmoidal weighting, and `EDM-monotonic' weighting. Similar to the observation on $\bepsilon$-prediction and EDM parameterized models, sigmoidal weighting and `EDM-monotonic' weighting worked slightly better than the original weighting.

\subsection{High resolution ImageNet} \label{sec:imagenet_big}

\begin{table}[t]
\caption{ImageNet 128x128 results. The first line corresponds to \emph{Simple Diffusion} model from \citet{hoogeboom2023simple} that serves as the baseline. We only changed the training noise schedule and weighting function; see Section \ref{sec:imagenet_big}.}
\label{table:imagenet128}
\scriptsize
\centering
{\renewcommand{\arraystretch}{1.2}%
\begin{tabular}{ lllcccc }
\toprule
& 
& 
& 
& \multicolumn{2}{c}{\textbf{FID} $\downarrow$}
& 
\\ 
\textbf{Model parameterization}
& \textbf{Training noise schedule}
& \textbf{Weighting function} 
& \textbf{Monotonic?}
& \textbf{train}
& \textbf{eval}
& \textbf{IS} $\uparrow$
\\ 
\midrule
$\rvv$-prediction 
& Cosine-shifted
& $\exp(-\lambda/2)$ (Baseline)
& \checkmark
& 1.91
& 3.23
& 171.9 $\pm$ 2.46 
\\
"
& Adaptive
& $\text{sigmoid}(-\lambda + 2)$-shifted
& \checkmark
& 1.91
& 3.41
& \textbf{183.1} $\pm$ 2.20
\\
"%
& Adaptive
& EDM-monotonic-shifted
& \checkmark
& \textbf{1.75}
& \textbf{2.88}
& 171.1 $\pm$ 2.67
\\

\bottomrule
\end{tabular}
}
\end{table}

In our final experiments, we tested whether the weighting functions that resulted in the best scores on ImageNet 64$\times$64, namely $\text{sigmoid}(-\lambda + 2)$ and `EDM-monotonic', also results in competitive scores on high-resolution generation. As baseline we use the \emph{Simple Diffusion} model from \cite{hoogeboom2023simple}, which reported the best FID scores to date on high-resolution ImageNet without sampling modifications (e.g. guidance).

We recruited the large U-ViT model from Simple Diffusion~\citep{hoogeboom2023simple}, and changed the training noise schedule and weighting function to our proposed ones. See Table 3 for the comparison to the baseline. Note that for higher-resolution models, \citet{hoogeboom2023simple} proposed a shifted version of the cosine noise schedule (Table \ref{table:noise_schedules}), that leads to a shifted version of the weighting function $w(\lambda)$. Similarly, we extended our proposed sigmoidal and `EDM-monotonic' weightings to their shifted versions (see Appendix \ref{sec:shifted_schedule} for details). For fair comparison, we adopted the same vanilla DDPM sampler as Simple Diffusion, without guidance or other advanced sampling techniques such as second-order sampler or rejection sampling. As shown in Table \ref{table:imagenet128}, with our adaptive noise schedule for training, the two weighting functions we proposed led to either better or comparable FID and IS scores on ImageNet 128$\times$128, compared to the baseline Simple Diffusion approach. 

\begin{table}[t]
\caption{Comparison to approaches in the literature for high-resolution ImageNet generation. `With guidance' indicates that the method was combined with classifier-free guidance~\citep{ho2022classifier}. $^\dagger$ Models under 'Latent diffusion with pretrained VAE' use the pre-trained VAE from Stable Diffusion~\citep{rombach2022high}, which used a much larger training corpus than the other models in this table.}
\label{table:imagenet128_comparison}
\scriptsize
\centering
{\renewcommand{\arraystretch}{1.2}%
\setlength\tabcolsep{4pt}
\begin{tabular}{ l|c|rrl|rrl }
\toprule
&
& \multicolumn{3}{l}{\textbf{Without guidance}}
& \multicolumn{3}{l}{\textbf{With guidance}}
\\
&
& \multicolumn{2}{c}{\textbf{FID} $\downarrow$}
& 
& \multicolumn{2}{c}{\textbf{FID} $\downarrow$}
& 
\\ 
\textbf{Method}
& 
& \textbf{train}
& \textbf{eval}
& \textbf{IS} $\uparrow$
& \textbf{train}
& \textbf{eval}
& \textbf{IS} $\uparrow$
\\ 
\midrule
\textbf{128 $\times$ 128 resolution} &&&\\
ADM~\citep{dhariwal2021diffusion}
&& 5.91 & &
& 2.97 & &
\\
CDM~\citep{ho2021cascaded}
&& 3.52
& 3.76
& 128.8 $\pm$ 2.5
\\
RIN~\citep{jabri2022scalable}
&& 2.75
&
& 144.1
\\
Simple Diffusion (U-Net)~\citep{hoogeboom2023simple}
&& 2.26
& \textbf{2.88}
& 137.3 $\pm$ 2.0
\\
Simple Diffusion (U-ViT, L)~\citep{hoogeboom2023simple}
&& 1.91
& 3.23
& 171.9 $\pm$ 2.5
& 2.05
& 3.57
& 189.9 $\pm$ 3.5
\\
\textbf{VDM++ (Ours)}, $w(\lambda)=\text{sigmoid}(-\lambda + 2)$
&& 1.91
& 3.41
& \textbf{183.1} $\pm$ 2.2
\\
\textbf{VDM++ (Ours)},  EDM-monotonic weighting
&& \textbf{1.75}
& \textbf{2.88}
& 171.1 $\pm$ 2.7
& \textbf{1.78}
& \textbf{3.16}
& \textbf{190.5} $\pm$ 2.3
\\
\midrule
\textbf{256 $\times$ 256 resolution} &&&\\
BigGAN-deep (no truncation)~\citep{brock2018large}
&& 6.9
& 
& 171.4 $\pm$ 2.0
\\
MaskGIT~\citep{chang2022maskgit} 
&& 6.18
& 
& 182.1
\\
ADM~\citep{dhariwal2021diffusion}
&& 10.94
&
&
& 3.94
& 
& 215.9
\\
CDM~\citep{ho2021cascaded}
&& 4.88
& 4.63
& 158.7 $\pm$ 2.3
\\
RIN~\citep{jabri2022scalable}
&& 3.42
& 
& 182.0
\\
Simple Diffusion (U-Net)~\citep{hoogeboom2023simple}
&& 3.76 
& 3.71
& 171.6 $\pm$ 3.1
\\
Simple Diffusion (U-ViT, L)~\citep{hoogeboom2023simple}
&& 2.77
& 3.75
& 211.8 $\pm$ 2.9
& 2.44
& 4.08
& 256.3 $\pm$ 5.0
\\
\textbf{VDM++ (Ours)}, EDM-monotonic weighting
&& \textbf{2.40}
& \textbf{3.36}
& \textbf{225.3} $\pm$ 3.2
& \textbf{2.12}
& \textbf{3.69}
& \textbf{267.7} $\pm$ 4.9
\\
\arrayrulecolor{gray}\midrule[0.1pt]\arrayrulecolor{black}
\textit{Latent diffusion with pretrained VAE:}
\\
DiT-XL/2~\citep{peebles2022scalable}
&& 9.62
&
& 121.5
& 2.27
&
& 278.2
\\
U-ViT~\citep{bao2023all}
&&
&
&
& 3.40
& 
& 
\\
Min-SNR-$\gamma$~\citep{hang2023efficient}
&&
&
&
& 2.06
&
& 
\\
MDT~\citep{gao2023masked}
&& 6.23
&
& 143.0
& 1.79
& 
& 283.0
\\
\midrule
\textbf{512 $\times$ 512 resolution} &&&\\
MaskGIT~\citep{chang2022maskgit}
&& 7.32
&
& 156.0
\\
ADM~\citep{dhariwal2021diffusion}
&& 23.24
& 
& 
& 3.85
& 
& 221.7
\\
RIN~\citep{jabri2022scalable}
&&
&
&
& 3.95
& 
& 216.0
\\
Simple Diffusion (U-Net)~\citep{hoogeboom2023simple}
&& 4.30
& 4.28
& 171.0 $\pm$ 3.0
\\
Simple Diffusion (U-ViT, L)~\citep{hoogeboom2023simple}
&& 3.54
& 4.53
& 205.3 $\pm$ 2.7
& 3.02
& 4.60
& 248.7 $\pm$ 3.4
\\
\textbf{VDM++ (Ours)}, EDM-monotonic weighting
&& \textbf{2.99}
& \textbf{4.09}
& \textbf{232.2} $\pm$ 4.2
& \textbf{2.65}
& \textbf{4.43}
& \textbf{278.1} $\pm$ 5.5
\\
\arrayrulecolor{gray}\midrule[0.1pt]\arrayrulecolor{black}
\textit{Latent diffusion with pretrained VAE:}
\\
DiT-XL/2~\citep{peebles2022scalable}
&& 12.03
& 
& 105.3
& 3.04
& 
& 240.8
\\
LDM-4~\citep{rombach2022high}
&& 10.56
& 
& 103.5 $\pm$ 1.2
& 3.60
& 
& 247.7 $\pm$ 5.6
\\
\bottomrule
\end{tabular}}
\end{table}

Next, we test our approach on ImageNet generation of multiple high resolutions (i.e., resolutions 128, 256 and 512), and compare with existing methods in the literature. See Table \ref{table:imagenet128_comparison} for the summary of quantitative evaluations and Figure \ref{fig:i128_sample} for some generated samples by our approach. With the shifted version of `EDM-monotonic' weighting, we achieved state-of-the-art FID and IS scores on all three resolutions of ImageNet generation among all approaches without guidance. With classifier-free guidance (CFG)~\citep{ho2022classifier}, our method outperforms all diffusion-based approaches on resolutions 128 and 512. On resolution 256, our method falls a bit behind \cite{gao2023masked} and \cite{hang2023efficient}, both of which were build upon the latent space of a pretrained auto-encoder from \emph{latent diffusion models}~\citep{rombach2022high} that was trained on much larger image datasets than ImageNet, while our model was trained on ImageNet dataset only. It is worth noting that we achieve significant improvements compared to Simple Diffusion which serves as the backbone of our method, on all resolutions and in both settings of with and without guidance. It is possible to apply our proposed weighting functions and adaptive noise schedules to other diffusion-based approaches such as~\citet{gao2023masked} to further improve their performance, which we shall leave to the future work.

\section{Conclusion and Discussion}

In summary, we have shown that the weighted diffusion loss, which generalizes diffusion objectives in the literature, has an interpretation as a weighted integral of ELBO objectives, with one ELBO per noise level. If the weighting function is monotonic, then we show that the objective has an interpretation as the ELBO objective with data augmentation, where the augmentation is noise perturbation with a distribution of noise levels.

Our results open up exciting new directions for future work. The newfound equivalence between monotonic weighting and the ELBO with data augmentation, allows for a direct apples-to-apples comparison of diffusion models with other likelihood-based models. For example, it allows one to optimize other likelihood-based models, such as autoregressive transformers, towards the same objective as monotonically weighted diffusion models. This would shine light on whether diffusion models are better or worse than other model types, as measured in terms of their held-out objectives as opposed to FID scores. We leave such interesting experiments to future work.

\section*{Acknowledgments}

We'd like to thank Alex Alemi and Ben Poole for fruitful discussions and feedback on early drafts. We thank Emiel Hoogeboom for advice and help on the implementation of Simple Diffusion. 

\bibliographystyle{plainnat}
\bibliography{references.bib} 

\begin{thebibliography}{52}
\providecommand{\natexlab}[1]{#1}
\providecommand{\url}[1]{\texttt{#1}}
\expandafter\ifx\csname urlstyle\endcsname\relax
  \providecommand{\doi}[1]{doi: #1}\else
  \providecommand{\doi}{doi: \begingroup \urlstyle{rm}\Url}\fi

\bibitem[Anderson(1982)]{anderson1982reverse}
Brian~DO Anderson.
\newblock Reverse-time diffusion equation models.
\newblock \emph{Stochastic Processes and their Applications}, 12\penalty0
  (3):\penalty0 313--326, 1982.

\bibitem[Bao et~al.(2023)Bao, Nie, Xue, Cao, Li, Su, and Zhu]{bao2023all}
Fan Bao, Shen Nie, Kaiwen Xue, Yue Cao, Chongxuan Li, Hang Su, and Jun Zhu.
\newblock All are worth words: A vit backbone for diffusion models.
\newblock In \emph{Proceedings of the IEEE/CVF Conference on Computer Vision
  and Pattern Recognition}, pages 22669--22679, 2023.

\bibitem[Brock et~al.(2018)Brock, Donahue, and Simonyan]{brock2018large}
Andrew Brock, Jeff Donahue, and Karen Simonyan.
\newblock Large scale gan training for high fidelity natural image synthesis.
\newblock \emph{arXiv preprint arXiv:1809.11096}, 2018.

\bibitem[Chang et~al.(2022)Chang, Zhang, Jiang, Liu, and
  Freeman]{chang2022maskgit}
Huiwen Chang, Han Zhang, Lu~Jiang, Ce~Liu, and William~T Freeman.
\newblock Maskgit: Masked generative image transformer.
\newblock In \emph{Proceedings of the IEEE/CVF Conference on Computer Vision
  and Pattern Recognition}, pages 11315--11325, 2022.

\bibitem[Chen et~al.(2021{\natexlab{a}})Chen, Zhang, Zen, Weiss, Norouzi, and
  Chan]{chen-iclr-2021}
Nanxin Chen, Yu~Zhang, Heiga Zen, Ron~J. Weiss, Mohammad Norouzi, and William
  Chan.
\newblock {WaveGrad: Estimating Gradients for Waveform Generation}.
\newblock In \emph{{ICLR}}, 2021{\natexlab{a}}.

\bibitem[Chen et~al.(2021{\natexlab{b}})Chen, Zhang, Zen, Weiss, Norouzi,
  Dehak, and Chan]{chen-interspeech-2021}
Nanxin Chen, Yu~Zhang, Heiga Zen, Ron~J. Weiss, Mohammad Norouzi, Najim Dehak,
  and William Chan.
\newblock {WaveGrad 2: Iterative Refinement for Text-to-Speech Synthesis }.
\newblock In \emph{{INTERSPEECH}}, 2021{\natexlab{b}}.

\bibitem[Child et~al.(2019)Child, Gray, Radford, and
  Sutskever]{child2019generating}
Rewon Child, Scott Gray, Alec Radford, and Ilya Sutskever.
\newblock Generating long sequences with sparse transformers.
\newblock \emph{arXiv preprint arXiv:1904.10509}, 2019.

\bibitem[Choi et~al.(2022)Choi, Lee, Shin, Kim, Kim, and
  Yoon]{choi2022perception}
Jooyoung Choi, Jungbeom Lee, Chaehun Shin, Sungwon Kim, Hyunwoo Kim, and
  Sungroh Yoon.
\newblock Perception prioritized training of diffusion models.
\newblock In \emph{Proceedings of the IEEE/CVF Conference on Computer Vision
  and Pattern Recognition}, pages 11472--11481, 2022.

\bibitem[Delbracio and Milanfar(2023)]{delbracio2023inversion}
Mauricio Delbracio and Peyman Milanfar.
\newblock Inversion by direct iteration: An alternative to denoising diffusion
  for image restoration.
\newblock \emph{arXiv preprint arXiv:2303.11435}, 2023.

\bibitem[Dhariwal and Nichol(2022)]{dhariwal2021diffusion}
Prafulla Dhariwal and Alex Nichol.
\newblock Diffusion models beat gans on image synthesis.
\newblock In \emph{{NeurIPS}}, 2022.

\bibitem[Ding et~al.(2021)Ding, Yang, Hong, Zheng, Zhou, Yin, Lin, Zou, Shao,
  Yang, and Tang]{CogView}
Ming Ding, Zhuoyi Yang, Wenyi Hong, Wendi Zheng, Chang Zhou, Da~Yin, Junyang
  Lin, Xu~Zou, Zhou Shao, Hongxia Yang, and Jie Tang.
\newblock Cogview: Mastering text-to-image generation via transformers, 2021.
\newblock URL \url{https://arxiv.org/abs/2105.13290}.

\bibitem[Gao et~al.(2023)Gao, Zhou, Cheng, and Yan]{gao2023masked}
Shanghua Gao, Pan Zhou, Ming-Ming Cheng, and Shuicheng Yan.
\newblock Masked diffusion transformer is a strong image synthesizer.
\newblock \emph{arXiv preprint arXiv:2303.14389}, 2023.

\bibitem[Hang et~al.(2023)Hang, Gu, Li, Bao, Chen, Hu, Geng, and
  Guo]{hang2023efficient}
Tiankai Hang, Shuyang Gu, Chen Li, Jianmin Bao, Dong Chen, Han Hu, Xin Geng,
  and Baining Guo.
\newblock Efficient diffusion training via min-snr weighting strategy.
\newblock \emph{arXiv preprint arXiv:2303.09556}, 2023.

\bibitem[Heusel et~al.(2017)Heusel, Ramsauer, Unterthiner, Nessler, and
  Hochreiter]{heusel2017gans}
Martin Heusel, Hubert Ramsauer, Thomas Unterthiner, Bernhard Nessler, and Sepp
  Hochreiter.
\newblock Gans trained by a two time-scale update rule converge to a local nash
  equilibrium.
\newblock \emph{arXiv preprint arXiv:1706.08500}, 2017.

\bibitem[Ho and Salimans(2022)]{ho2022classifier}
Jonathan Ho and Tim Salimans.
\newblock Classifier-free diffusion guidance.
\newblock \emph{arXiv preprint arXiv:2207.12598}, 2022.

\bibitem[Ho et~al.(2020)Ho, Jain, and Abbeel]{ho2020denoising}
Jonathan Ho, Ajay Jain, and Pieter Abbeel.
\newblock Denoising diffusion probabilistic models.
\newblock \emph{arXiv preprint arXiv:2006.11239}, 2020.

\bibitem[Ho et~al.(2022)Ho, Saharia, Chan, Fleet, Norouzi, and
  Salimans]{ho2021cascaded}
Jonathan Ho, Chitwan Saharia, William Chan, David~J Fleet, Mohammad Norouzi,
  and Tim Salimans.
\newblock Cascaded diffusion models for high fidelity image generation.
\newblock \emph{{JMLR}}, 2022.

\bibitem[Hoogeboom et~al.(2023)Hoogeboom, Heek, and
  Salimans]{hoogeboom2023simple}
Emiel Hoogeboom, Jonathan Heek, and Tim Salimans.
\newblock simple diffusion: End-to-end diffusion for high resolution images.
\newblock \emph{arXiv preprint arXiv:2301.11093}, 2023.

\bibitem[Huang et~al.(2021)Huang, Lim, and Courville]{huang2021variational}
Chin-Wei Huang, Jae~Hyun Lim, and Aaron Courville.
\newblock A variational perspective on diffusion-based generative models and
  score matching.
\newblock \emph{arXiv preprint arXiv:2106.02808}, 2021.

\bibitem[Jabri et~al.(2022)Jabri, Fleet, and Chen]{jabri2022scalable}
Allan Jabri, David Fleet, and Ting Chen.
\newblock Scalable adaptive computation for iterative generation.
\newblock \emph{arXiv preprint arXiv:2212.11972}, 2022.

\bibitem[Karras et~al.(2022)Karras, Aittala, Aila, and
  Laine]{karras2022elucidating}
Tero Karras, Miika Aittala, Timo Aila, and Samuli Laine.
\newblock Elucidating the design space of diffusion-based generative models.
\newblock \emph{arXiv preprint arXiv:2206.00364}, 2022.

\bibitem[Kingma et~al.(2021)Kingma, Salimans, Poole, and
  Ho]{kingma2021variational}
Diederik Kingma, Tim Salimans, Ben Poole, and Jonathan Ho.
\newblock Variational diffusion models.
\newblock \emph{Advances in neural information processing systems},
  34:\penalty0 21696--21707, 2021.

\bibitem[Kingma and Ba(2014)]{kingma2014adam}
Diederik~P Kingma and Jimmy Ba.
\newblock Adam: A method for stochastic optimization.
\newblock \emph{arXiv preprint arXiv:1412.6980}, 2014.

\bibitem[Kingma and Dhariwal(2018)]{kingma2018glow}
Diederik~P Kingma and Prafulla Dhariwal.
\newblock Glow: Generative flow with invertible 1x1 convolutions.
\newblock \emph{arXiv preprint arXiv:1807.03039}, 2018.

\bibitem[Kingma and Welling(2013)]{kingma2013auto}
Diederik~P Kingma and Max Welling.
\newblock Auto-encoding variational {B}ayes.
\newblock \emph{International Conference on Learning Representations}, 2013.

\bibitem[Kong et~al.(2021)Kong, Ping, Huang, Zhao, and
  Catanzaro]{kong-arxiv-2020}
Zhifeng Kong, Wei Ping, Jiaji Huang, Kexin Zhao, and Bryan Catanzaro.
\newblock {DiffWave: A Versatile Diffusion Model for Audio Synthesis}.
\newblock In \emph{{ICLR}}, 2021.

\bibitem[Le et~al.(2023)Le, Vyas, Shi, Karrer, Sari, Moritz, Williamson,
  Manohar, Adi, Mahadeokar, et~al.]{le2023voicebox}
Matthew Le, Apoorv Vyas, Bowen Shi, Brian Karrer, Leda Sari, Rashel Moritz,
  Mary Williamson, Vimal Manohar, Yossi Adi, Jay Mahadeokar, et~al.
\newblock Voicebox: Text-guided multilingual universal speech generation at
  scale.
\newblock \emph{arXiv preprint arXiv:2306.15687}, 2023.

\bibitem[Lipman et~al.(2022)Lipman, Chen, Ben-Hamu, Nickel, and
  Le]{lipman2022flow}
Yaron Lipman, Ricky~TQ Chen, Heli Ben-Hamu, Maximilian Nickel, and Matt Le.
\newblock Flow matching for generative modeling.
\newblock \emph{arXiv preprint arXiv:2210.02747}, 2022.

\bibitem[Lyu(2012)]{lyu2012interpretation}
Siwei Lyu.
\newblock Interpretation and generalization of score matching.
\newblock \emph{arXiv preprint arXiv:1205.2629}, 2012.

\bibitem[Meng et~al.(2021)Meng, Song, Song, Zhao, and Ermon]{meng2021improved}
Chenlin Meng, Jiaming Song, Yang Song, Shengjia Zhao, and Stefano Ermon.
\newblock Improved autoregressive modeling with distribution smoothing.
\newblock \emph{arXiv preprint arXiv:2103.15089}, 2021.

\bibitem[Nichol and Dhariwal(2021)]{nichol2021improved}
Alex Nichol and Prafulla Dhariwal.
\newblock Improved denoising diffusion probabilistic models.
\newblock \emph{arXiv preprint arXiv:2102.09672}, 2021.

\bibitem[Nichol et~al.(2021)Nichol, Dhariwal, Ramesh, Shyam, Pamela~Mishkin,
  Sutskever, and Chen]{nichol-glide}
Alex Nichol, Prafulla Dhariwal, Aditya Ramesh, Pranav Shyam, Bob~McGrew
  Pamela~Mishkin, Ilya Sutskever, and Mark Chen.
\newblock {GLIDE: Towards Photorealistic Image Generation and Editing with
  Text-Guided Diffusion Models}.
\newblock In \emph{{arXiv:2112.10741}}, 2021.

\bibitem[Peebles and Xie(2022)]{peebles2022scalable}
William Peebles and Saining Xie.
\newblock Scalable diffusion models with transformers.
\newblock \emph{arXiv preprint arXiv:2212.09748}, 2022.

\bibitem[Poole et~al.(2022)Poole, Jain, Barron, and
  Mildenhall]{poole2022dreamfusion}
Ben Poole, Ajay Jain, Jonathan~T. Barron, and Ben Mildenhall.
\newblock {DreamFusion: Text-to-3D using 2D Diffusion}.
\newblock \emph{arXiv}, 2022.

\bibitem[Ramesh et~al.(2022)Ramesh, Dhariwal, Nichol, Chu, and
  Chen]{ramesh-dalle2}
Aditya Ramesh, Prafulla Dhariwal, Alex Nichol, Casey Chu, and Mark Chen.
\newblock {Hierarchical Text-Conditional Image Generation with CLIP Latents}.
\newblock In \emph{{arXiv}}, 2022.

\bibitem[Rezende et~al.(2014)Rezende, Mohamed, and
  Wierstra]{rezende2014stochastic}
Danilo~J Rezende, Shakir Mohamed, and Daan Wierstra.
\newblock Stochastic backpropagation and approximate inference in deep
  generative models.
\newblock In \emph{International Conference on Machine Learning}, pages
  1278--1286, 2014.

\bibitem[Rombach et~al.(2022{\natexlab{a}})Rombach, Blattmann, Lorenz, Esser,
  and Ommer]{rombach2022high}
Robin Rombach, Andreas Blattmann, Dominik Lorenz, Patrick Esser, and Bj{\"o}rn
  Ommer.
\newblock High-resolution image synthesis with latent diffusion models.
\newblock In \emph{Proceedings of the IEEE/CVF conference on computer vision
  and pattern recognition}, pages 10684--10695, 2022{\natexlab{a}}.

\bibitem[Rombach et~al.(2022{\natexlab{b}})Rombach, Blattmann, Lorenz, Esser,
  and Ommer]{rombach-cvpr-2022}
Robin Rombach, Andreas Blattmann, Dominik Lorenz, Patrick Esser, and Björn
  Ommer.
\newblock {High-Resolution Image Synthesis with Latent Diffusion Models}.
\newblock In \emph{{CVPR}}, 2022{\natexlab{b}}.

\bibitem[Saharia et~al.(2022{\natexlab{a}})Saharia, Chan, Chang, Lee, Ho,
  Salimans, Fleet, and Norouzi]{sahariac-palette}
Chitwan Saharia, William Chan, Huiwen Chang, Chris~A. Lee, Jonathan Ho, Tim
  Salimans, David~J. Fleet, and Mohammad Norouzi.
\newblock {Palette: Image-to-Image Diffusion Models}.
\newblock In \emph{{SIGGRAPH}}, 2022{\natexlab{a}}.

\bibitem[Saharia et~al.(2022{\natexlab{b}})Saharia, Chan, Saxena, Li, Whang,
  Denton, Ghasemipour, Ayan, Mahdavi, Lopes, Salimans, Ho, Fleet, and
  Norouzi]{sahariac-imagen}
Chitwan Saharia, William Chan, Saurabh Saxena, Lala Li, Jay Whang, Emily
  Denton, Seyed Kamyar~Seyed Ghasemipour, Burcu~Karagol Ayan, S.~Sara Mahdavi,
  Rapha~Gontijo Lopes, Tim Salimans, Jonathan Ho, David~J Fleet, and Mohammad
  Norouzi.
\newblock {Photorealistic Text-to-Image Diffusion Models with Deep Language
  Understanding}.
\newblock In \emph{{NeurIPS}}, 2022{\natexlab{b}}.

\bibitem[Saharia et~al.(2022{\natexlab{c}})Saharia, Ho, Chan, Salimans, Fleet,
  and Norouzi]{saharia2021image}
Chitwan Saharia, Jonathan Ho, William Chan, Tim Salimans, David~J Fleet, and
  Mohammad Norouzi.
\newblock Image super-resolution via iterative refinement.
\newblock \emph{{IEEE Transactions on Pattern Analysis and Machine
  Intelligence}}, 2022{\natexlab{c}}.

\bibitem[Salimans and Ho(2022)]{salimans2022progressive}
Tim Salimans and Jonathan Ho.
\newblock Progressive distillation for fast sampling of diffusion models.
\newblock \emph{arXiv preprint arXiv:2202.00512}, 2022.

\bibitem[Salimans et~al.(2016)Salimans, Goodfellow, Zaremba, Cheung, Radford,
  and Chen]{salimans2016improved}
Tim Salimans, Ian Goodfellow, Wojciech Zaremba, Vicki Cheung, Alec Radford, and
  Xi~Chen.
\newblock Improved techniques for training gans.
\newblock \emph{Advances in neural information processing systems}, 29, 2016.

\bibitem[Sohl-Dickstein et~al.(2015)Sohl-Dickstein, Weiss, Maheswaranathan, and
  Ganguli]{sohl2015deep}
Jascha Sohl-Dickstein, Eric Weiss, Niru Maheswaranathan, and Surya Ganguli.
\newblock Deep unsupervised learning using nonequilibrium thermodynamics.
\newblock In \emph{International Conference on Machine Learning}, pages
  2256--2265, 2015.

\bibitem[Song and Ermon(2019)]{song2019generative}
Yang Song and Stefano Ermon.
\newblock Generative modeling by estimating gradients of the data distribution.
\newblock In \emph{Advances in Neural Information Processing Systems}, pages
  11895--11907, 2019.

\bibitem[Song et~al.(2021{\natexlab{a}})Song, Durkan, Murray, and
  Ermon]{song2021maximum}
Yang Song, Conor Durkan, Iain Murray, and Stefano Ermon.
\newblock Maximum likelihood training of score-based diffusion models.
\newblock \emph{arXiv e-prints}, pages arXiv--2101, 2021{\natexlab{a}}.

\bibitem[Song et~al.(2021{\natexlab{b}})Song, Sohl-Dickstein, Kingma, Kumar,
  Ermon, and Poole]{song2020score}
Yang Song, Jascha Sohl-Dickstein, Diederik~P Kingma, Abhishek Kumar, Stefano
  Ermon, and Ben Poole.
\newblock Score-{B}ased {G}enerative {M}odeling {T}hrough {S}tochastic
  {D}ifferential {E}quations.
\newblock In \emph{International Conference on Learning Representations},
  2021{\natexlab{b}}.

\bibitem[Vahdat et~al.(2021)Vahdat, Kreis, and Kautz]{vahdat2021score}
Arash Vahdat, Karsten Kreis, and Jan Kautz.
\newblock Score-based generative modeling in latent space.
\newblock \emph{arXiv preprint arXiv:2106.05931}, 2021.

\bibitem[Vincent(2011)]{vincent2011connection}
Pascal Vincent.
\newblock A connection between score matching and denoising autoencoders.
\newblock \emph{Neural computation}, 23\penalty0 (7):\penalty0 1661--1674,
  2011.

\bibitem[Watson et~al.(2022)Watson, Chan, Ho, Tagliasacchi, and
  Norouzi]{watson20223dim}
Daniel Watson, Ricardo Chan, William Martin-Brualla, Jonathan Ho, Andrea
  Tagliasacchi, and Mohammad Norouzi.
\newblock {Novel View Synthesis with Diffusion Models}.
\newblock \emph{arXiv}, 2022.

\bibitem[Whang et~al.(2022)Whang, Delbracio, Talebi, Saharia, Dimakis, and
  Milanfar]{whang-cvpr-2022}
Jay Whang, Mauricio Delbracio, Hossein Talebi, Chitwan Saharia, Alexandros~G.
  Dimakis, and Peyman Milanfar.
\newblock {Deblurring via Stochastic Refinement}.
\newblock In \emph{{CVPR}}, 2022.

\bibitem[Yu et~al.(2022)Yu, Xu, Koh, Luong, Baid, Wang, Vasudevan, Ku, Yang,
  Ayan, Hutchinson, Han, Parekh, Li, Zhang, and Jason~Baldridge]{yu-parti-2022}
Jiahui Yu, Yuanzhong Xu, Jing~Yu Koh, Thang Luong, Gunjan Baid, Zirui Wang,
  Vijay Vasudevan, Alexander Ku, Yinfei Yang, Burcu~Karagol Ayan, Ben
  Hutchinson, Wei Han, Zarana Parekh, Xin Li, Han Zhang, and Yonghui~Wu
  Jason~Baldridge.
\newblock {Scaling Autoregressive Models for Content-Rich Text-to-Image
  Generation}.
\newblock In \emph{{arXiv:2206.10789}}, 2022.

\end{thebibliography}

\newpage

\appendix 

\section{Main proof}
\label{app:main_proof}

Here we'll provide a proof of \Eqref{eq:mainresult}. 

Note that like in the main text, we use shorthand notation:
\begin{align}
\Lt{t} := \Ltfull{t,...,1}
\end{align}

\subsection{Time derivative of $\Ltfull{t,...,1}$}
\label{app:time_derivative_of_dlt}
Let $dt$ denote an infinitesimal change in time.
Note that $\Lt{t-dt}$ can be decomposed as the sum of a KL divergence and an expected KL divergence:
\begin{align}
\Lt{t-dt} 
&= \Lt{t} + \E_{q(\rvz_t|\rvx)}[ D_{KL}(q(\rvz_{t-dt}|\rvz_t,\rvx)||p(\rvz_{t-dt}|\rvz_t)]
\label{eq:bla192}
\end{align}
Due to this identity, the time derivative $d/dt\, \Lt{t}$ can be expressed as:
\begin{align}
\frac{d}{dt} \Lt{t} 
&= \frac{1}{dt} (\Lt{t}-\Lt{t-dt})
\\
&= - \frac{1}{dt} \E_{q(\rvz_t|\rvx)}[ D_{KL}(q(\rvz_{t-dt}|\rvz_t,\rvx)||p(\rvz_{t-dt}|\rvz_t)]
\end{align}
In Appendix E of \citep{kingma2021variational}, it is shown that in our model, this equals:
\begin{align}
\frac{d}{dt} \Lt{t} 
&= - \frac{1}{2} \frac{SNR(t-dt) - SNR(t)}{dt} ||\rvx - \hat{\rvx}_{\bT}(\rvz_t;\lambda_t)||_2^2
\\
&= \frac{1}{2} SNR'(t) ||\rvx - \hat{\rvx}_{\bT}(\rvz_t;\lambda_t)||_2^2
\end{align}
where $\rvz_t = \alpha_\lambda \rvx + \sigma_\lambda \bepsilon$, and $SNR(t) := \exp(\lambda)$ in our model, and $SNR'(t) = d/dt\,SNR(t) = e^{\lambda} \,d\lambda/dt$, so in terms of our definition of $\lambda$, this is:
\begin{align}
\frac{d}{dt} \Lt{t} 
&= \frac{1}{2} e^{\lambda} \frac{d\lambda}{dt}
\E_{\bepsilon \sim \mathcal{N}(0, \mathbf{I})}\left[||\rvx - \hat{\rvx}_{\bT}(\rvz_t;\lambda_t)||_2^2 \right]
\label{eq:ddtltz}
\end{align}
In terms of $\bepsilon$-prediction (see Section \ref{sec:param}), because $||\bepsilon - \hat{\bepsilon}_{\bT}||_2^2 = e^{\lambda} 
||\rvx - \hat{\rvx}_{\bT}||_2^2$ this simplifies to:
\begin{empheq}[box={\mymath}]{equation}
\frac{d}{dt} \Lt{t} = \frac{1}{2} \frac{d\lambda}{dt} 
\E_{\bepsilon \sim \mathcal{N}(0, \mathbf{I})}\left[||\bepsilon - \hat{\bepsilon}_{\bT}(\rvz_t;\lambda_t)||_2^2 \right]
\label{eq:ddtlt}
\end{empheq}
where $\rvz_\lambda = \alpha_\lambda \rvx + \sigma_\lambda \bepsilon$. This can be easily translated to other parameterizations; see \ref{sec:param}. 

This allows us to rewrite the weighted loss of \Eqref{eq:lw_new} as:
\begin{align}
\mathcal{L}_w(\rvx) 
&=
\frac{1}{2} 
\E_{t \sim \mathcal{U}(0,1), \bepsilon \sim \mathcal{N}(0,\mathbf{I})} 
\left[
w(\lambda_t)
\cdot
-\frac{d\lambda}{dt}
\cdot || 
\hat{\bepsilon}_{\bT}(\rvz_t;\lambda) - \bepsilon
||_2^2
\right]
\\
&=
\E_{t \sim \mathcal{U}(0,1))} 
\left[
w(\lambda_t)
\cdot
-
\frac{1}{2} 
\frac{d\lambda}{dt}
\E_{\bepsilon \sim \mathcal{N}(0,\mathbf{I})} 
\left[
|| 
\hat{\bepsilon}_{\bT}(\rvz_t;\lambda) - \bepsilon
||_2^2
\right]
\right]
\\
&=
\E_{t \sim \mathcal{U}(0,1))} 
\left[
-
\frac{d}{dt} \Lt{t}
\,w(\lambda_t)
\right]
\\
&=
\int_0^1
-
\frac{d}{dt} \Lt{t}
\,w(\lambda_t)
\,dt
\label{eq:lw3}
\end{align}

\subsection{Integration by parts}
\label{sec:intbyparts}

Integration by parts is a basic identity, which tells us that:
\begin{align*}
- \int_a^b f(t) g'(t) dt  &=  \int_a^b f'(t)g(t) dt + f(a)g(a) - f(b)g(b)
\end{align*}

This allows us to further rewrite the expression of the weighted loss in \Eqref{eq:lw3} as:
\begin{align}
\mathcal{L}_w(\rvx) 
=&
\int_0^1
-
\frac{d}{dt} \Lt{t}
\,w(\lambda_t)
\;dt
\\
=&
\int_0^1
\frac{d}{dt} w(\lambda_t)
\,
\Lt{t}
\;dt
+ w(\lmax) \Lt{0}
- w(\lmin) \Lt{1}
\label{eq:intbyparts}
\end{align}
where
\begin{align}
- w(\lmin)\, \Lt{1} = - w(\lmin)\, \Ltfull{1}
\end{align}
is constant w.r.t. $\bT$, since it does not involve the score function, and typically very small, since $\Ltfull{1}$ is typically small by design.

The term $w(\lmax) \Ltfull{0,...,1}$ is typically small, since $w(\lmax)$ is typically very small (see Figure \ref{fig:weighting_functions}).

This concludes our proof of \Eqref{eq:mainresult}.
\renewcommand{\qedsymbol}{$\blacksquare$}
\qed

\section{Visualization}
We tried to create a helpful visualization of the result from Section \ref{sec:intbyparts}.
Note that we can rewrite:
\begin{align}
\int_0^1
\frac{d}{dt} w(\lambda_t)
\,
\Lt{t}
\;dt
= \int_{t=1}^{t=0} w(\lambda_t) \, d\Lt{t}
\end{align}

The relationship in \Eqref{eq:intbyparts} can be rewritten as:
\begin{align}
w(\lmin)\, \mathcal{L}(1) + \int_{t=1}^{t=0} w(\lambda_t) \, d\Lt{t}
= w(\lmax)\, \mathcal{L}(0) + \int_{t=0}^{t=1} \Lt{t} \, dw(\lambda_t)
\end{align}
The first LHS term equals a weighted prior loss term, and the second LHS term equals the weighted diffusion loss. From a geometric perspective, the two LHS terms together define an area that equals the area given by the right term, as illustrated in the figure below.
\tcbincludegraphics[width=1.0\textwidth,colback=white,colframe=gray]{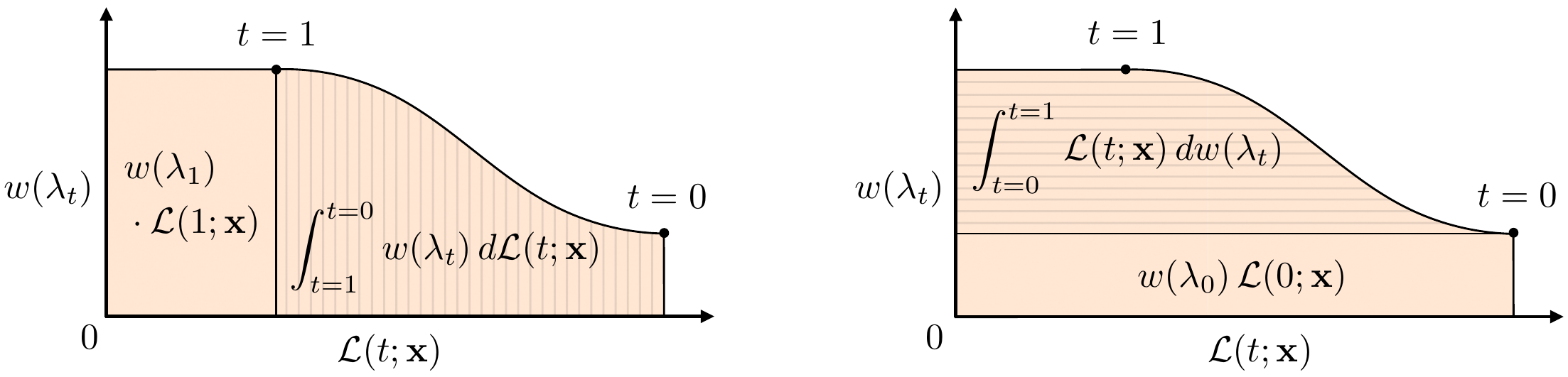}

On the left, we have a rectangular area that equals a weighted prior loss $w(\lmin)\mathcal{L}(1)$, plus a curved area equal to the weighted diffusion loss $\int_{t=1}^{t=0} w(\lambda_t) \, d\Lt{t}$. This integral can be intuitively understood as a Riemann sum over many tiny rectangles going from left ($t=1$) to right ($t=0$), each with height $w(\lambda_t)$ and width $d\Lt{t}$. On the right, we have the same total area, but divided up into two different subareas: a rectangular area $w(\lmax) \mathcal{L}(0)$ and a curved area that equals the integral $\int_{t=0}^{t=1} \Lt{t} \, dw(\lambda_t)$ going upwards from $t=0$ to $t=1$, which can also be intuitively understood as another Riemann sum, with each tiny rectangle having width $\Lt{t}$ and height $dw(\lambda_t)$. The area of each of those tiny rectangles on the right can be understood as the ELBO at each noise level, $\Lt{t}$, times the weight of the ELBO at each noise level, $dw(\lambda_t)$.

\section{Relationship between $\Ltfull{t,...,1}$ and the ELBO}
\label{sec:elbo}
First, note that:
\begin{align}
\Lt{t} = \Ltfull{t,...,1} \geq D_{KL}(q(\rvz_t|\rvx)||p(\rvz_t))
\end{align}
More precisely, the joint KL divergence $\Ltfull{t,...,1}$ is the expected negative ELBO of noise-perturbed data, plus a constant entropy term:
\begin{align}
\Lt{t} = \Ltfull{t,...,1} = - \E_{q(\rvz_t|\rvx)}[ \text{ELBO}_t(\rvz_t)] - \underbrace{\mathcal{H}(q(\rvz_t|\rvx))}_{\text{constant}}
\end{align}
where the ELBO of noise-perturbed data is:
\begin{align}
\text{ELBO}_t(\rvz_t) 
&:= \E_{q(\tilde{\rvz}_t|\rvz_t)}[ \log p(\rvz_t, \tilde{\rvz}_t) - \log q(\tilde{\rvz}_t|\rvz_t) ]
\\
&\leq \log p(\rvz_t)
\end{align}
where $\tilde{\rvz}_t := \rvz_{t+dt,...,1}$.

So, $\Lt{t}$ is the expected negative ELBO of noise-perturbed data $\rvz_t$:
\begin{align}
\Lt{t} 
= - \E_{q(\rvz_t|\rvx)}[ \text{ELBO}_t(\rvz_t)] + \text{constant}
\geq - \E_{q(\rvz_t|\rvx)}[ \log p(\rvz_t)] + \text{constant}
\end{align}

Therefore, the expression of the weighted loss in \Eqref{eq:mainresult2} can be rewritten as:
\begin{align}
\mathcal{L}_w(\rvx)
&=
\E_{p_w(t)} 
\left[
\Lt{t}
\right]
+ \text{constant}
\\
&= 
- \underbrace{
\E_{p_w(t), q(\rvz_t|\rvx)} 
\left[
\text{ELBO}_t(\rvz_t)
\right]
}_{\text{ELBO of noise-perturbed data}}
\;+\; \text{constant}
\\
&\geq 
- \underbrace{
\E_{p_w(t), q(\rvz_t|\rvx)} 
\left[
\log p(\rvz_t)
\right]
}_{\text{Log-likelihood of noise-perturbed data}}
\;+\; \text{constant}
\label{eq:distaug}
\end{align}
where $w(\lmin)$ is constant w.r.t. the diffusion model parameters. Therefore, minimizing $\mathcal{L}_w(\rvx)$ is equivalent to maximizing this expected ELBO of noise-perturbed data.

\section{Derivation of weighting functions for previous works}
\label{sec:loss_functions}

The loss function used in previous works are equivalent to the weighted loss with a certain choice of noise schedule and weighting function $w(\lambda)$. In this section, we derive these weighting functions $w(\lambda)$.

\subsection{`Elucidating Diffusion Models' (EDM) \citep{karras2022elucidating}}
\label{sec:edm}
\citet{karras2022elucidating} proposes the following training objective:
\begin{align}
  \mathcal{L_{\rm edm}}(\rvx) = \E_{\tilde{\sigma} \sim p(\tilde{\sigma}), \bepsilon \sim \mathcal{N}(0, \rmI)}\left[\tilde{w}(\tilde{\sigma})\| \rvx - \hat{\rvx}_\bT(\rvz_{\tilde{\sigma}}; \tilde{\sigma}) \|_2^2\right],
\end{align}
where $p(\tilde{\sigma})$ and $\tilde{w}(\tilde{\sigma})$ are defined as:
\begin{align}
p(\log \tilde{\sigma}) &= \mathcal{N}(\log \tilde{\sigma}; P_{\rm mean}, P^2_{\rm std}), \\
    \tilde{w}(\tilde{\sigma}) & = (\tilde{\sigma}^2 + \tilde{\sigma}_{\rm data}^2) / (\tilde{\sigma}^2 \cdot \tilde{\sigma}^2_{\rm data}),
\end{align}
where $\tilde{\sigma}$ is equivalent to the standard deviation of noise added to the clean sample $\rvx$ in the VP SDE case, so $\tilde{\sigma}^2 = e^{-\lambda}$. \cite{karras2022elucidating} used hyperparameters $P_{\rm mean} = -1.2$, $P_{\rm std} = 1.2$ and $\tilde{\sigma}_{\rm data} = 0.5$.

This can be rewritten in terms of $\lambda = - 2 \log \tilde{\sigma}$ as:
\begin{align}
  \mathcal{L_{\rm edm}}(\rvx) = \E_{p(\lambda), \bepsilon \sim \mathcal{N}(0, \rmI)}\left[\tilde{w}(\lambda)\| \rvx - \hat{\rvx}_\bT(\rvz_\lambda; \lambda) \|_2^2\right],
\end{align}
where $p(\lambda)$ and $\tilde{w}(\lambda)$ are defined as 
\begin{align}
p(\lambda) &= \mathcal{N}(\lambda; 2.4, 2.4^2), \\
    \tilde{w}(\lambda) & = (e^{-\lambda} + 0.5^2) / (e^{-\lambda} \cdot 0.5^2),
\label{eq:loss-edm}\end{align}
Comparing \Eqref{eq:loss-edm} with the weighted loss expressed in terms of $\rvx$-prediction parameterization (Section \ref{sec:param}), we see that (ignoring the constant scaling factor $1/2$) the EDM objective is a special case of the weighted loss with weighting function: 
\begin{align}
    w(\lambda) &= p(\lambda) e^{-\lambda} \cdot \tilde{w}(\lambda) \\
    &= p(\lambda) (e^{-\lambda} + 0.5^2) / 0.5^2.
\end{align}
where the divison by $\tilde{\sigma}_{\rm data}^2$ can be ignored since it's constant. This leads to:
\begin{align}
w(\lambda) =  \mathcal{N}(\lambda; 2.4, 2.4^2) (e^{-\lambda} + 0.5^2)
\end{align}

\subsection{The $\rvv$-prediction loss / `SNR+1'-weighting \citep{salimans2022progressive}}
\label{sec:vparam}

\cite{salimans2022progressive} introduced the $\rvv$-parameterization, with a $\rvv$-prediction model $\hat{\rvv}$, where:
\begin{align}
\rvv &:= \alpha_\lambda \bepsilon - \sigma_\lambda \rvx
\\
\hat{\rvv} &:= \alpha_\lambda \hat{\bepsilon} - \sigma_\lambda \hat{\rvx}
\end{align}
They propose to minimize a $\rvv$-prediction loss, $\E[||\rvv - \hat{\rvv}||_2^2]$.
Note that $\hat{\bepsilon} = (\rvz_\lambda - \alpha_\lambda \hat{\rvx})/\sigma_\lambda$.
For our general family, this implies:
\begin{align}
||\rvv - \hat{\rvv}||_2^2 
&= \sigma_\lambda^2 (e^{\lambda} + 1)^2 ||\rvx - \hat{\rvx}||_2^2
\label{eq:v-loss-general}
\\
&= \alpha_\lambda^2 (e^{-\lambda} + 1)^2  ||\bepsilon - \hat{\bepsilon}||_2^2
\label{eq:v-loss-general-2}
\end{align}
In the special case of the variance preserving (VP) SDE, this simplifies to:
\begin{align}
||\rvv - \hat{\rvv}||_2^2 
&= (e^{\lambda} + 1) ||\rvx - \hat{\rvx}||_2^2
\\
&= (e^{-\lambda} + 1) ||\bepsilon - \hat{\bepsilon}||_2^2.
\end{align}

Since the $\bepsilon$-prediction loss corresponds to minimizing the weighted loss with $w(\lambda)=p(\lambda)$, the $\rvv$-prediction loss corresponds to minimizing the weighted loss with $w(\lambda)=(e^{-\lambda} + 1)p(\lambda)$.

Note that \cite{salimans2022progressive} view the loss from the $\rvx$-prediction viewpoint, instead of our ELBO viewpoint; so in their view, minimizing simply $||\rvx - \hat{\rvx}||_2$ means no weighting. Note that $e^{\lambda}$ is the signal-to-noise ratio (SNR). Since $||\bepsilon - \hat{\bepsilon}||_2^2 = e^\lambda ||\rvx - \hat{\rvx}||_2^2$, they call the $\bepsilon$-prediction loss 'SNR weighting', and since $||\rvv - \hat{\rvv}||_2^2 
= (e^{\lambda} + 1) ||\rvx - \hat{\rvx}||_2^2$, they call this 'SNR+1'-weighting.

\cite{salimans2022progressive} propose to use optimize a VP SDE with a cosine schedule $p(\lambda) = \text{sech}(\lambda/2)/(2\pi) = 1/(2 \pi \cosh(-\lambda/2))$ and the $\rvv$-prediction loss: $\E[||\rvv - \hat{\rvv}||_2^2 ]$. This corresponds to minimizing the weighted loss with:
\begin{align}
w(\lambda)
&= (e^{-\lambda} + 1) p(\lambda) \\
&= (e^{-\lambda} + 1) / (2 \pi \cosh(-\lambda/2)) \\
&= \pi e^{-\lambda/2}
\end{align}
The factor $\pi$ can be ignored since it's constant, so we can equivalently use:
\begin{align}
w(\lambda) = e^{-\lambda/2}
\end{align}

\subsubsection{With shifted cosine schedule} \label{sec:shifted_schedule}
\citep{hoogeboom2023simple} extended the cosine schedule to a shifted version: $p(\lambda) = \text{sech}(\lambda/2 - s)/(2\pi)$, where $s = \log(64/d)$, where 64 is the base resolution and $d$ is the model resolution (e.g. 128, 256, 512, etc.). In this case the weighting is:
\begin{align}
w(\lambda)
&= (e^{-\lambda} + 1) p(\lambda) \\
&= (2/\pi) e^{-s} e^{-\lambda/2}
\end{align}
Since $(2/\pi) e^{-s}$ is constant w.r.t $\lambda$, the weighting is equivalent to the weighting for the unshifted cosine schedule.

\subsection{Flow Matching with the Optimal Transport flow path (FM-OT)~\citep{lipman2022flow}}
\label{sec:fm}

Flow Matching \citep{lipman2022flow} with the Optimal Transport flow path can be seen as a special case of Gaussian diffusion with the weighted loss.
\subsubsection{Noise schedule}
Note that in \citep{lipman2022flow}, time goes from 1 to 0 as we go forward in time. Here, we'll let time go from 0 to 1 as we go forward in time, consistent with the rest of this paper. We'll also assume $\sigma_0=0$, for which we can later correct by truncation (see Section \ref{sec:truncation}). In this model, the forward process $q(\rvz_t|\rvx)$ is defined by:
\begin{align}
\rvz_t 
&= \alpha_t \rvx + \sigma_t \bepsilon\\
&= (1-t) \rvx + t \bepsilon
\end{align}
This implies that the log-SNR is given by:
\begin{align}
\lambda_t = \fl(t)
&= \log(\alpha^2_t/\sigma^2_t)\\
&= 2 \log ( (1-t) / t)
\end{align}
Its inverse is given by:
\begin{align}
t 
&= \fli(\lambda) = 1/(1+e^{\lambda/2})\\
&= \text{sigmoid}(-\lambda/2)
\end{align}
The derivative, as a function of $t$, is:
\begin{align}
\frac{d\lambda}{dt} = \frac{d}{dt} \fl(t) = 2/(-t + t^2)
\end{align}
This derivative of its inverse, as a function of $\lambda$, is:
\begin{align}
\frac{dt}{d\lambda} = \frac{d}{d\lambda} \fli(\lambda) = \frac{d}{d\lambda} \text{sigmoid}(-\lambda/2) 
= - \text{sech}^2(\lambda/4)/8
\end{align}

The corresponding density is
\begin{align}
p(\lambda) = - \frac{d}{d\lambda} \fli(\lambda) = \text{sech}^2(\lambda/4)/8
\end{align}
which is a Logistic distribution; see also Table \ref{table:noise_schedules}.

\subsubsection{Score function parameterization and loss function}
\cite{lipman2022flow} then propose the following generative model ODE:
\begin{align}
d\rvz = - \hat{\rvo}(\rvz_t,t) dt
\end{align}
The model is then optimized with the \emph{Conditional flow matching} (CFM) loss:
\begin{align}
\mathcal{L}_{\text{CFM}}(\rvx) = \E_{t \sim \mathcal{U}(0,1), \bepsilon \sim \mathcal{N}(0,\bfI)}[ || \rvo - \hat{\rvo}||_2^2]
\end{align}
where they use the parameterization:
\begin{align}
\rvo &:= \rvx - \bepsilon
\end{align}

\subsubsection{Weighting function}
What is the weighting function $w(\lambda)$ corresponding to this loss? Note that this parameterization means that:
\begin{align}
\rvz_t &= (1-t)\rvx + t\bepsilon\\
&= (1-t)\rvo + \bepsilon\\
\rvo &= (\rvz_t - \bepsilon)/(1-t)
\end{align}
Since $t = 1/(1+e^{\lambda/2})$, we have that $1/(1-t) = 1+e^{-\lambda/2}$, so parameterized as a function of $\lambda$, we have:
\begin{align}
\rvo &= (\rvz_\lambda - \bepsilon) (1+e^{-\lambda/2})
\end{align}

Likewise, we can parameterize $\rvo$-prediction in terms of $\bepsilon$-prediction:
\begin{align}
\hat{\rvo}(\rvz_\lambda, \lambda) &= (\rvz_\lambda - \hat{\bepsilon}(\rvz_\lambda, \lambda))(1+e^{-\lambda/2})
\end{align}

We can translate the $\rvo$-prediction loss to a $\bepsilon$-prediction loss:
\begin{align}
|| \rvo - \hat{\rvo}(\rvz_\lambda, \lambda)||_2^2
&= (1+e^{-\lambda/2})^2 || \bepsilon - \hat{\bepsilon}(\rvz_\lambda, \lambda)||_2^2
\end{align}

Therefore, combining the derivations above, the CFM loss, formulated in terms of the $\lambda$ parameterization instead of $t$, and in terms of the $\bepsilon$-prediction parameterization instead of the $\rvo$-prediction parameterization, is:
\begin{align}
\mathcal{L}_{\text{CFM}}(\rvx) 
&= \E_{t \sim \mathcal{U}(0,1), \bepsilon \sim \mathcal{N}(0,\bfI)}[ || \rvo - \hat{\rvo}(\rvz_t,t)||_2^2]\\
&= \int_0^1 \E_{\bepsilon \sim \mathcal{N}(0,\bfI)}[ || \rvo - \hat{\rvo}(\rvz_t,t)||_2^2] \;dt\\
&= \int_{\lmin}^{\lmax} - \frac{dt}{d\lambda}\E_{\bepsilon \sim \mathcal{N}(0,\bfI)}[ || \rvo - \hat{\rvo}(\rvz_\lambda,\lambda)||_2^2] \;d\lambda\\
&= \int_{\lmin}^{\lmax} (\text{sech}^2(\lambda/4)/8) \E_{\bepsilon \sim \mathcal{N}(0,\bfI)}[ || \rvo - \hat{\rvo}(\rvz_\lambda,\lambda)||_2^2] \;d\lambda\\
&= \int_{\lmin}^{\lmax} (\text{sech}^2(\lambda/4)/8) (1+e^{-\lambda/2})^2 \E_{\bepsilon \sim \mathcal{N}(0,\bfI)}[ || \bepsilon - \hat{\bepsilon}(\rvz_\lambda, \lambda)||_2^2] \;d\lambda\\
&= \frac{1}{2} \int_{\lmin}^{\lmax} w(\lambda) \E_{\bepsilon \sim \mathcal{N}(0,\bfI)}[ || \bepsilon - \hat{\bepsilon}(\rvz_\lambda, \lambda)||_2^2] \;d\lambda\\
&= \frac{1}{2} \E_{\bepsilon \sim \mathcal{N}(0,\bfI), \lambda \sim \tilde{p}(\lambda)}\left[ \frac{w(\lambda)}{\tilde{p}(\lambda)} || \bepsilon - \hat{\bepsilon}(\rvz_\lambda, \lambda)||_2^2\right]
\end{align}
where $\tilde{p}$ is any distribution with full support on $[\lmin,\lmax]$, and
where:
\begin{align}
w(\lambda) 
&= 2 (\text{sech}^2(\lambda/4)/8) (1+e^{-\lambda/2})^2\\
&= e^{-\lambda/2}
\end{align}
Therefore, this weighting is equivalent to the weighting for the $\rvv$-prediction loss with cosine schedule (Section \ref{sec:vparam}): the CFM loss is equivalent to the $\rvv$-prediction loss with cosine schedule.

\subsection{Inversion by Direct Iteration (InDI)~\citep{delbracio2023inversion}}
\label{sec:indi}

\cite{delbracio2023inversion} propose Inversion by Direct Iteration (InDI). Their forward process is identical to the forward process of FM-OT \citep{lipman2022flow} introduced in Section \ref{sec:fm}:
\begin{align}
\rvz_t &= (1-t) \rvx + t \bepsilon
\end{align}
As derived in Section \ref{sec:fm} above, this means that the distribution over log-SNR $\lambda$ is the Logistic distribution: $p(\lambda) = \text{sech}^2(\lambda/4)/8$.
The proposed loss function is the $\rvx$-prediction loss:
\begin{align}
\mathcal{L}_{\text{InDI}}(\rvx) = \E_{t \sim \mathcal{U}(0,1), \bepsilon \sim \mathcal{N}(0,\bfI)}[ || \rvx - \hat{\rvx}(\rvz_t,t)||_2^2]
\end{align}
Since $||\rvx - \hat{\rvx}||_2^2 = e^{-\lambda} ||\bepsilon - \hat{\bepsilon}||_2^2$, and the $\bepsilon$-prediction loss corresponds to minimizing the weighted loss with $w(\lambda)=p(\lambda)$, the $\rvx$-prediction loss above corresponds to minimizing the weighted loss with:
\begin{align}
w(\lambda) 
&= e^{-\lambda} p(\lambda)\\
&= e^{-\lambda} \text{sech}^2(\lambda/4)/8
\end{align}
Which is a slightly different weighting then the FM-OT weighting, giving a bit more weighting to lower noise levels.

\subsection{Perception prioritized weighting (P2 weighting)~\citep{choi2022perception}}
\label{sec:p2-weighting}
\citet{choi2022perception} proposed a new weighting function:
\begin{equation}
    w(\lambda) = \frac{-dt/d\lambda}{(k + e^\lambda)^\gamma} = \frac{p(\lambda)}{(k + e^\lambda)^\gamma},
\end{equation}
where empirically they set $k=1$ and $\gamma$ as either $0.5$ or $1$. Compared to the $\epsilon$-prediction objective, where $w(\lambda) = dt/d\lambda = p(\lambda)$, this objective put more emphasis on the middle regime of the whole noise schedule, which \citet{choi2022perception} hypothesized to be the most important regime for creating content that is sensitive to visual perception. When combined with the most commonly used cosine noise schedule~\citep{nichol2021improved}, the weighting function becomes $w(\lambda) = \text{sech}(\lambda/2)/(1 + e^\lambda)^\gamma$.
 
\subsection{Min-SNR-$\gamma$ weighting~\citep{hang2023efficient}}
\label{sec:min-snr}
\citet{hang2023efficient} proposed the following training objective:
\begin{align}
    \mathcal{L}_{\text{MinSNR}}(\rvx) &= \E_{t \sim \mathcal{U}(0,1), \bepsilon \sim \mathcal{N}(0,\bfI)} \left[ \min\{e^\lambda, \gamma\} \|\rvx - \hat{\rvx}(\rvz_t;\lambda)\|_2^2 \right] \\
    &= E_{t \sim \mathcal{U}(0,1), \bepsilon \sim \mathcal{N}(0,\bfI)} \left[ \min\{1, \gamma e^{-\lambda}\} \|\bepsilon - \hat{\bepsilon}(\rvz_t;\lambda)\|_2^2 \right] \\
    &= E_{t \sim \mathcal{U}(0,1), \bepsilon \sim \mathcal{N}(0,\bfI)} \left[ \min\{1, \gamma e^{-\lambda}\}\cdot -\frac{dt}{d\lambda} \cdot -\frac{d\lambda}{dt} \|\bepsilon - \hat{\bepsilon}(\rvz_t;\lambda)\|_2^2 \right] \\
    &= E_{t \sim \mathcal{U}(0,1), \bepsilon \sim \mathcal{N}(0,\bfI)} \left[ \min\{1, \gamma e^{-\lambda}\}p(\lambda) \cdot -\frac{d\lambda}{dt} \|\bepsilon - \hat{\bepsilon}(\rvz_t;\lambda)\|_2^2 \right].
\end{align}
Therefore, it corresponds to $w(\lambda) = \min\{1, \gamma e^{-\lambda}\}p(\lambda)$. The motivation of the work is to avoid the model focusing too much on small noise levels, since it shares similar hypothesis to~\citep{choi2022perception} that small noise levels are responsible for cleaning up details that may not be perceptible. A cosine noise schedule is then combined with the proposed weighting function, leading to $w(\lambda) = \text{sech}(\lambda/2) \cdot \min\{1, \gamma e^{-\lambda}\}$. $\gamma$ is set as $5$ empirically. 

\section{Useful Equations}

\subsection{SDEs}
\label{sec:sdes}

The forward process is a Gaussian diffusion process, whose time evolution is described by a stochastic differential equation (SDE):
\begin{align}
d\rvz = \underbrace{\rvf(\rvz,t)}_{\text{drift}} dt + \underbrace{g(t)}_{\text{diffusion}} d\rvw
\label{eq:sde1}
\end{align}

For derivations of diffusion SDEs, see Appendix B of \citep{song2020score}. Their $\beta(t)$ equals $\frac{d}{dt} \log(1 + e^{-\lambda_t})$ in our formulation, and their $\int_0^t \beta(s) ds$ equals $\log(1 + e^{-\lambda_t})$, where they assume that $\lambda \to \infty$ at $t=0$.

\subsubsection{Variance-preserving (VP) SDE}
A common choice is the variance-preserving (VP) SDE, which generalizes denoising diffusion models~\citep{ho2020denoising} to continuous time~\citep{song2020score,kingma2021variational}.  In the VP case:
\begin{align}
\rvf(\rvz,t) &= -\frac{1}{2}\left(\frac{d}{dt} \log(1 + e^{-\lambda_t})\right) \rvz\\
g(t)^2 &= \frac{d}{dt} \log(1 + e^{-\lambda_t})\\
\alpha^2_\lambda &= \text{sigmoid}(\lambda)
\\
\sigma^2_\lambda &= \text{sigmoid}(-\lambda)
\\
p(\rvz_1) &= \mathcal{N}(0,\bfI)
\end{align}

\subsubsection{Variance-exploding (VE) SDE}

Another common choice of the variance-exploding (VE) SDE. In the VE case:
\begin{align}
\rvf(\rvz,t) &= 0\\
g(t)^2 &= \frac{d}{dt} \log(1 + e^{-\lambda_t})\\
\alpha^2_\lambda &= 1\\
\sigma^2_\lambda &= e^{-\lambda}
\\
p(\rvz_1) &= \mathcal{N}(0,e^{-\lmin} \bfI)
\end{align}

\subsection{Possible parameterizations of the score network}
\label{sec:param}

There are various ways of parameterizing the score network:
\begin{align}
\snT(\rvz; \lambda) &= - \nabla_{\rvz} E_{\bT}(\rvz, \lambda)
& \text{\;(With the gradient of an energy-based model)}
\\
&= - \hat{\bepsilon}_{\bT}(\rvz;\lambda) / \sigma_\lambda
& \text{\;(With a noise prediction model)}
\\
&= - \sigma_\lambda^{-2}(\rvz - \alpha_\lambda \hat{\rvx}_{\bT}(\rvz;\lambda))
& \text{\;(With a data prediction model)}
\end{align}
We can let a neural network output any of $\snT(\rvz_\lambda; \lambda)$, $\hat{\bepsilon}_{\bT}$ or $\hat{\rvx}_{\bT}$, and we can convert the variables to each other using the equalities above.

The chosen relationship between $\rvz_\lambda$, $\hat{\rvx}$, $\hat{\bepsilon}$ and $\snT(\rvz_\lambda; \lambda)$ above, are due to the following relationships between $\rvz_\lambda$, $\rvx$ and $\bepsilon$:
\begin{align}
    \rvz_\lambda &= \alpha_\lambda \rvx + \sigma_\lambda \bepsilon\\
    \rvx &= \alpha_\lambda^{-1} (\rvz_\lambda - \sigma_\lambda \bepsilon)\\
    \bepsilon &= \sigma_\lambda^{-1} (\rvz_\lambda - \alpha_\lambda \rvx)
\end{align}
And:
\begin{align}
\nabla_{\rvz_\lambda} \log q(\rvz_\lambda | \rvx)
    &= \nabla_{\rvz_\lambda} -||\rvz_\lambda - \alpha_\lambda \rvx||_2^2/(2\sigma^2_\lambda)\\
    &= -\sigma^{-2}_\lambda(\rvz_\lambda - \alpha_\lambda \rvx)\\
    &= -\sigma^{-2}_\lambda(\alpha_\lambda \rvx + \sigma_\lambda \bepsilon - \alpha_\lambda \rvx)\\
    &= -\bepsilon/\sigma_\lambda
\label{eq:grad_qz}
\end{align}

In addition, there's the $\rvv$-prediction parameterization ($\rvv := \alpha_\lambda \bepsilon - \sigma_\lambda \rvx$) explained in \ref{sec:vparam}, and the $\rvo$-prediction parameterization ($\rvo := \rvx - \bepsilon$) explained in \ref{sec:fm}.

\cite{karras2022elucidating} proposed a specific $\rmF$-parametrization, with an $\rmF$-prediction model $\hat{\rmF}_\theta$. In the special case of variance explosion (VE) SDE, it is formulated as:
\begin{align}
\rvx &= \frac{\tilde{\sigma}_{\rm data}^2}{e^{-\lambda} + \tilde{\sigma}_{\rm data}^2} \rvz_\lambda + \frac{e^{-\lambda/2} \tilde{\sigma}_{\rm data}}{\sqrt{e^{-\lambda} + \tilde{\sigma}_{\rm data}^2}} \rmF
\end{align}
where $\tilde{\sigma}_{\rm data} = 0.5$. Generalizing this to our more general family with arbitrary drift, this corresponds to:
\begin{align}
\rvx &= \frac{\tilde{\sigma}_{\rm data}^2 \alpha_\lambda}{e^{-\lambda} + \tilde{\sigma}_{\rm data}^2} \rvz_\lambda + \frac{e^{-\lambda/2} \tilde{\sigma}_{\rm data}}{\sqrt{e^{-\lambda} + \tilde{\sigma}_{\rm data}^2}} \rmF
\end{align}
So that we have:
\begin{align}
    \rmF &= \frac{\sqrt{e^{-\lambda} + \tilde{\sigma}_{\rm data}^2}}{e^{-\lambda / 2}\tilde{\sigma}_{\rm data}} \rvx - \frac{\tilde{\sigma}_{\rm data}\alpha_\lambda}{e^{-\lambda/2}\sqrt{e^{-\lambda} + \tilde{\sigma}_{\rm data}^2}} \rvz_\lambda \\
    &= -\frac{\sqrt{e^{-\lambda} + \tilde{\sigma}_{\rm data}^2}}{\tilde{\sigma}_{\rm data}} \bepsilon + \frac{e^{\lambda / 2}(e^{-\lambda} + \tilde{\sigma}_{\rm data}^2 - \tilde{\sigma}_{\rm data}^2 \alpha_\lambda^2)}{\sqrt{e^{-\lambda} + \tilde{\sigma}_{\rm data}^2} \tilde{\sigma}_{\rm data}\alpha_\lambda}\rvz_\lambda 
\end{align}
In summary, given these different parameterizations, the $\bepsilon$-prediction loss can be written in terms of other parameterizations as follows:
\begin{align}
||\bepsilon - \hat{\bepsilon}_{\bT}||_2^2 
&= 
e^{\lambda} 
||\rvx - \hat{\rvx}_{\bT}||_2^2
&\text{\;\;($\bepsilon$-prediction and $\rvx$-prediction error)}
\\
&= 
 \sigma_\lambda^2
||\nabla_{\rvz_\lambda} \log q(\rvz_\lambda | \rvx) - \snT||_2^2
&\text{\;\;(score prediction)}
\\
&= 
\alpha_\lambda^{-2} (e^{-\lambda} + 1)^{-2}
||\rvv - \hat{\rvv}_{\bT}||_2^2
&\text{\;\;($\rvv$-prediction, general)}
\\
&= 
(e^{-\lambda} + 1)^{-1}
||\rvv - \hat{\rvv}_{\bT}||_2^2
&\text{\;\;($\rvv$-prediction with VP SDE)}
\\
&= 
(e^{-\lambda} / \tilde{\sigma}_{\rm data}^2 + 1)^{-1}
||\rmF - \hat{\rmF}_{\bT}||_2^2
&\text{\;\;($\rmF$-prediction)}
\end{align}
Interestingly, if we set $\tilde{\sigma}_{\rm data}^2 = 1$, the training objectives of $\rmF$-prediction and $\rvv$-prediction are the same. %

\subsection{Noise schedules}
\label{sec:schedules}

\begin{table}[t]
\caption{
Noise schedules used in our experiments: cosine \citep{nichol2021improved}, shifted cosine \citep{hoogeboom2023simple}, and EDM \citep{karras2022elucidating} training and sampling schedules. Note that these are the noise schedules \emph{before} truncation (Section \ref{sec:truncation}).
}
\label{table:noise_schedules}
\scriptsize
{\renewcommand{\arraystretch}{1.5}%
\begin{tabular}{ llll }
\toprule
\textit{Noise schedule name} & $\lambda = \fl(t) = ...$ & $t = \fli(\lambda) = ...$ & $p(\lambda) = - \frac{d}{d\lambda} \fli(\lambda) = ...$\\
\midrule
Cosine & $-2\log(\tan(\pi t/2))$ & $(2/\pi) \arctan(e^{-\lambda/2})$ & $\text{sech}(\lambda/2)/(2\pi)$\\ 
Shifted cosine & $-2\log(\tan(\pi t/2)) + 2 s$ & $(2/\pi) \arctan(e^{-\lambda/2 - s})$ & $\text{sech}(\lambda/2 - s)/(2\pi)$\\
EDM (training) & $-F_{\mathcal{N}}^{-1}(t; 2.4, 2.4^2)$ & $F_{\mathcal{N}}(-\lambda; 2.4, 2.4^2)$ & $\mathcal{N}(\lambda; 2.4, 2.4^2)$\\
EDM (sampling) 
& {\tiny ${-2 \rho \log(\sigma_{\max}^{1/\rho} \atop + (1-t) (\sigma_{\min}^{1/\rho} - \sigma_{\max}^{1/\rho}))}$}
& $1 - \frac{e^{-\lambda / (2\rho)} - \sigma_{\max}^{1/\rho}}{ \sigma_{\min}^{1/\rho} - \sigma_{\max}^{1/\rho}}$
& $\frac{e^{-\lambda / (2\rho)}}{2 \rho (\sigma_{\max}^{1/\rho} - \sigma_{\min}^{1/\rho})}$ \\
Flow Matching with OT (see \ref{sec:fm})
& $2 \log((1-t)/t)$
& $1/(1+e^{\lambda/2})$
& $\text{sech}^2(\lambda/4)/8$
\\
\bottomrule
\end{tabular}
}
\end{table}

During model training, we sample time $t$ uniformly: $t \sim \mathcal{U}(0,1)$, then compute $\lambda = \fl(t)$. This results in a distribution over noise levels $p(\lambda)$, whose cumulative density function (CDF) is given by $1-\fli(\lambda)$. For $\lambda \in [\lmin, \lmax]$ the probability density function (PDF) is the derivative of the CDF, which is $p(\lambda_t) = -(d/d\lambda)\, \fli(\lambda) = -dt/d\lambda = -1/\fl'(t)$. Outside of $[\lmin, \lmax]$ the probability density is 0.

In table \ref{table:noise_schedules} we provide some popular noise schedules: cosine \citep{nichol2021improved}, shifted cosine \citep{hoogeboom2023simple}, and EDM \citep{karras2022elucidating}. We do not list the 'linear' schedule by \citep{ho2020denoising}, $\fl(t) = -\log(e^{t^2}-1)$ which has fallen out of use.

Note that these are the noise schedules \emph{before} truncation. The truncation procedure is given in \ref{sec:truncation}.

where:
\begin{itemize}
\item In the shifted cosine schedule, $s = \log(64/d)$, where 64 is the base resolution and $d$ is the used resolution (e.g. 128, 256, 512, etc.).
\item In the EDM training schedule, the function $F_{\mathcal{N}}(\lambda; \mu, \sigma^2)$ is the Normal distribution CDF, and $\mathcal{N}(\lambda; \mu, \sigma^2)$ is its PDF.
\item In the EDM sampling schedule, $\rho=7$, $\sigma_{\min} = 0.002$, $\sigma_{\max} = 80$. The density function $p(\lambda)$ in the table has support $\lambda \in [- \log \sigma^2_{\max}, - \log \sigma^2_{\min}]$. Outside this range, $p(\lambda) = 0$.
\end{itemize}

\subsubsection{Truncation}
\label{sec:truncation}

The noise schedules above are truncated, resulting in a noise schedule $\tfl(0)$ whose endpoints have desired values $[\tfl(0),\tfl(1)] = [\lmax, \lmin]$:
\begin{align}
\tfl(t) &:= \fl(t_0 + (t_1 - t_0) t)\\
\text{where:}\;\; 
t_0 &:= \fli(\lmax)\\
t_1 &:= \fli(\lmin)
\end{align}
Its inverse is:
\begin{align}
\tfli(\lambda) = (\fli(\lambda) - t_0)/(t_1 - t_0)
\end{align}
And the corresponding probability density:
\begin{align}
\text{if\;} \lmin \leq \lambda \leq \lmax:\;\;\;& \tilde{p}(\lambda) = -\frac{d}{d\lambda} \tfli(\lambda) 
= -\frac{d}{d\lambda} \fli(\lambda) / (t_1 - t_0)
= p(\lambda) / (t_1 - t_0)\\
\text{else:}\;\;\;& \tilde{p}(\lambda) = 0
\end{align}

\subsection{Sampling}
\label{sec:sampling}

\cite{anderson1982reverse} showed that if $\snT(\rvz;\lambda) = \nabla_{\rvz} \log q_t(\rvz)$, then the forward SDE is exactly reversed by the following SDE:
\begin{align}
d\rvz = [\rvf(\rvz,t) - g(t)^2 \snT(\rvz;\lambda)] dt + g(t) d\rvw
\label{eq:reverse}
\end{align}

Recent diffusion models have used increasingly sophisticated samplers. As an alternative to solving the SDE, \citep{song2020score} showed that sampling from the model can alternatively be done by solving the following \emph{probability flow} ODE:
\begin{align}
d\rvz = [\rvf(\rvz,t) - \frac{1}{2} g(t)^2 \snT(\rvz;\lambda)] dt
\label{eq:reverse_ode}
\end{align}
which, under the assumption that $\snT$ is a conservative vector field, will result in the same marginals $p(\rvz_t)$ as the SDE of \Eqref{eq:reverse} for every $t \in [0,1]$, and therefore also the same marginal $p(\rvx)$. 

Note that due to the continuous-time nature of the model, any sampling method is necessarily approximate, with the discretization error depending on various factors including the choice of noise schedule. For sampling we can therefore typically use a different noise schedule $\fl$ for sampling than for training, and we can change the SDE drift term; as long as we appropriately rescale the input to the score network, this would still result in correct samples. 

\section{Adaptive noise schedule}
\label{sec:adaptive}

\begin{figure}[t]
	\centering
	\begin{subfigure}{.49\textwidth}
		\includegraphics[width=.9\textwidth]{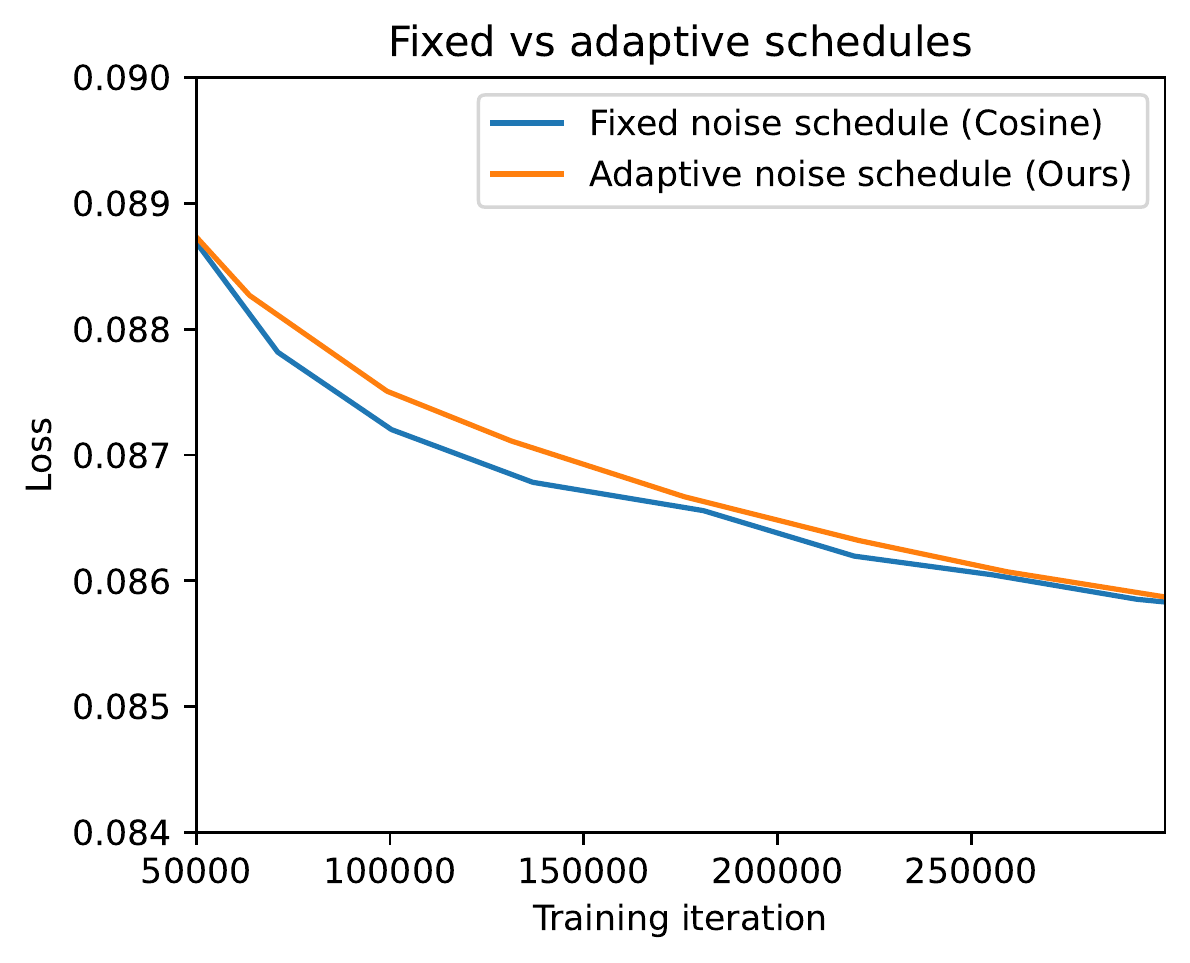}
		\caption{}
	\end{subfigure}
	\begin{subfigure}{.49\textwidth}
		\includegraphics[width=.9\textwidth]{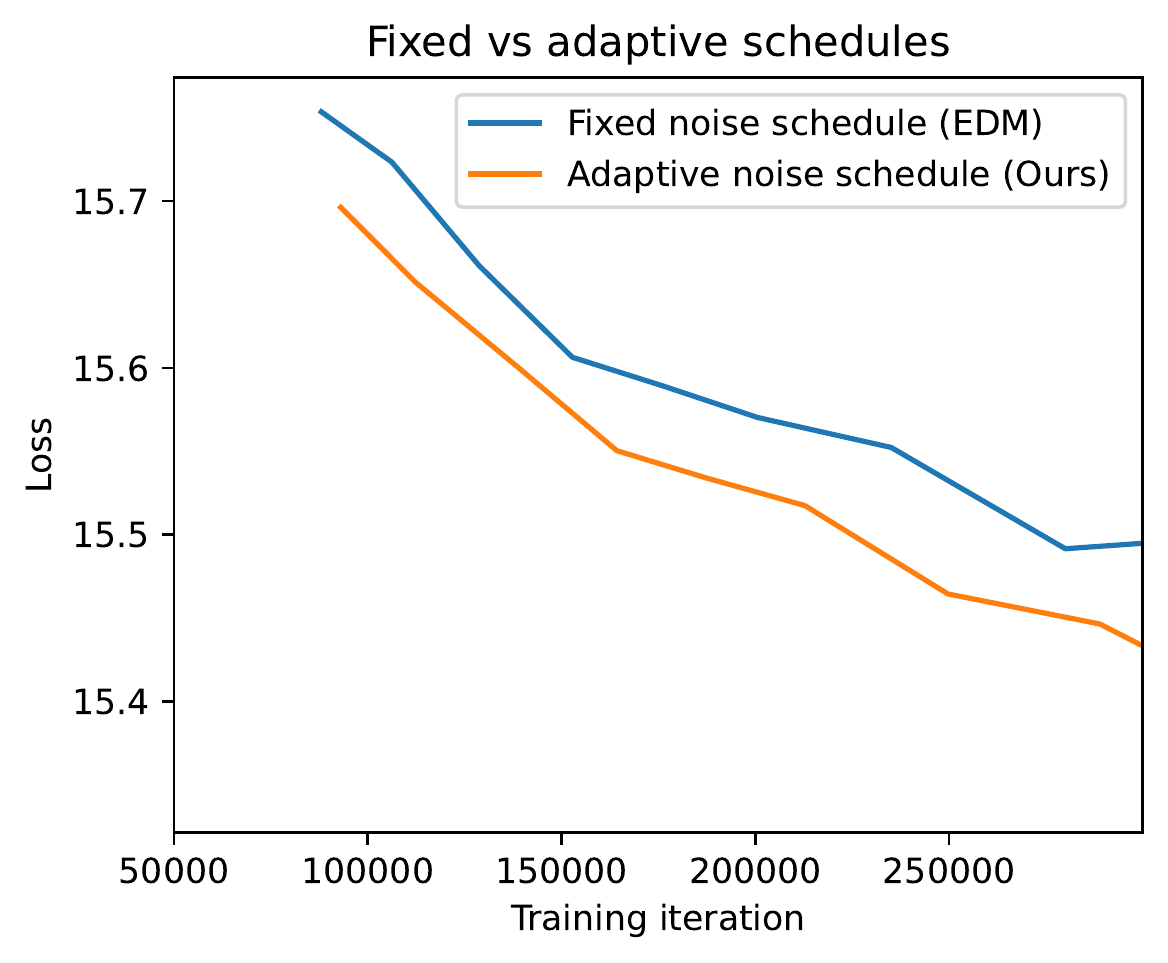}
		\caption{}
	\end{subfigure}
\caption{Our proposed adaptive noise schedule allows us to freely change the weighting function without needing to handtune a corresponding noise schedule. For some models it leads to slightly worse convergence speed (left) compared to a static noise schedule, probably because the noise schedule was already well-tuned for the weighting function, while in other cases it led to faster convergence (right). The adaptive noise schedule did not significantly affect the end result, but it did allow us to more freely experiment with weighting functions.}
\label{fig:adaptive}
\end{figure}

\label{sec:adaptive_main}

The invariance shown in Section \ref{sec:invariance} holds for the loss $\mathcal{L}_w(\rvx)$, but not for the Monte Carlo estimator of the loss that we use in training, based on random samples $t \sim \mathcal{U}(0,1), \bepsilon \sim \mathcal{N}(0,\mathbf{I})$. The noise schedule still affects the \emph{variance} of this Monte Carlo estimator and its gradients; therefore, the noise schedule affects the efficiency of optimization. 

In fact, the noise schedule acts as an importance sampling distribution for estimating the loss integral of \Eqref{eq:lw_integral}. 
Specifically, note that $p(\lambda) = -1/(d\lambda/dt)$. We can therefore rewrite the weighted loss as the following, which clarifies the role of $p(\lambda)$ as an importance sampling distribution:
\begin{align}
\mathcal{L}_w(\rvx) = 
\frac{1}{2} 
\E_{\bepsilon \sim \mathcal{N}(0,\mathbf{I}), \lambda \sim p(\lambda)} 
\left[
\frac{\colorD{w(\lambda)}}
{\colorA{p(\lambda)}}
|| 
\colorC{\hat{\bepsilon}_{\bT}(\rvz_\lambda;\lambda)} - \colorE{\bepsilon}
||_2^2
\right]
\end{align}

In order to avoid having to hand-tune the noise schedule for different weighting functions, we implemented an adaptive noise schedule. The noise schedule $\lambda_t$ is updated online, where we let $p(\lambda) \propto \E_{\rvx \sim \mathcal{D},\bepsilon \sim \mathcal{N}(0, \mathbf{I})}[\colorD{w(\lambda)}
|| 
\colorC{\hat{\bepsilon}_{\bT}(\rvz_\lambda;\lambda)} - \colorE{\bepsilon}
||_2^2 ]$. This noise schedule ensures that the loss is spread evenly over time, i.e. that the magnitude of the loss $\E_{\rvx \sim \mathcal{D}, \bepsilon \sim \mathcal{N}(0,\mathbf{I})} 
\left[
(\colorD{w(\lambda)}/\colorA{p(\lambda)})
|| 
\colorC{\hat{\bepsilon}_{\bT}(\rvz_\lambda;\lambda)} - \colorE{\bepsilon}
||_2^2
\right]$ is approximately invariant to $\lambda$ or $t$. We find that this often significantly speeds op optimization.

We implemented an adaptive noise schedule $p(\lambda)$, where:
\begin{align}
p(\lambda) \propto \E_{\rvx \sim \mathcal{D},\bepsilon \sim \mathcal{N}(0, \mathbf{I})}[w(\lambda) ||\bepsilon - \hat{\bepsilon}_{\bT}(\rvz_\lambda;\lambda)||_2^2 ]
\label{eq:adaptive_lambda}
\end{align}

In practice we approximate this by dividing the range $[\lmin, \lmax]$ into 100 evenly spaced bins, and during training keep an exponential moving average (EMA) of $w(\lambda) ||\bepsilon - \hat{\bepsilon}_{\bT}(\rvz_\lambda;\lambda)||_2^2 $ within each bin. From these EMAs we construct a piecewise linear function $\fl(t)$ such that \Eqref{eq:adaptive_lambda} is approximately satisfied. The EMAs and corresponding noise schedule $p(\lambda)$ are updated at each training iteration.

In experiments we measure the effect of changing the fixed noise schedule of existing modules with an adaptive schedule. We found that this lead to approximately equal FID scores. In half of the experiments, optimization was approximately as fast as with the original noise schedule, while in the other half the adaptive noise schedule lead to faster optimization (see Figure \ref{fig:adaptive}). The end results were not significantly altered.

\section{Relationship between the KL divergence and Fisher divergence}
\label{sec:fisher}

We'll use the following definition of the Fisher divergence~\citep{lyu2012interpretation}:
\begin{align}
D_{F}(q(\rvx)||p(\rvx)) := \E_q(\rvx)[||\nabla_\rvx \log q(\rvx) - \nabla_\rvx \log p(\rvx)||_2^2]
\label{eq:df_def}
\end{align}
\begin{theorem}
Assume a model in the family specified in Section \ref{sec:model_family}, and assume the score network encodes a conservative vector field: $\rvs_{\bT}(\rvz_t, \lambda_t) = \nabla_{\rvz_t} \log p(\rvz_t)$ (not assumed by the other theorems). Then:
\begin{align}
\frac{d}{d\lambda} \Ltfull{t,...,1} 
&= \frac{1}{2} \sigma_\lambda^2
 D_{F}(q(\rvz_t|\rvx)||p(\rvz_t))
\label{eq:df2}
\end{align}
\label{theorem:df}
\end{theorem}

\begin{proof}[Proof of Theorem \ref{theorem:df}]

Note that (see \Eqref{eq:grad_qz}):
\begin{align}
\nabla_{\rvz_t} \log q(\rvz_t | \rvx)
&= -\bepsilon/\sigma_\lambda
\end{align}
And assume the score network encodes a conservative vector field:
\begin{align}
\nabla_{\rvz_t} \log p(\rvz_t)
= \rvs_{\bT}(\rvz_t, \lambda_t)
= -\hat{\bepsilon}_{\bT}(\rvz_t; t)/\sigma_\lambda
\label{eq:conservative_vector_field}
\end{align}
So the time derivative of \Eqref{eq:ddtlt} can be expressed as:
\begin{align}
\frac{d}{d\lambda} \Ltfull{t,...,1}
&= 
\frac{1}{2} \sigma_\lambda^2 
\E_{q(\rvz_t | \rvx)}\left[
||\nabla_{\rvz_t} \log q(\rvz_t | \rvx) - \nabla_{\rvz_t} \log p(\rvz_t)||_2^2
\right]
\end{align}
\Eqref{eq:df2} follows from the definition of the Fisher divergence.
\end{proof}

\subsection{Comparison with Theorem 1 by \cite{lyu2012interpretation}}

\cite{lyu2012interpretation} prove a similar result in their Theorem 1. We'll translate their result into our notation. In particular, let the forward process be as in our family, such that $q(\rvz_t|\rvx) = \mathcal{N}(\rvz_t; \alpha_\lambda \rvx, \sigma^2_t \bfI)$. The marginal (data) distribution is $q(\rvx)$, such that $q(\rvz_t) = \int q(\rvz_t | \rvx) q(\rvx) d\rvx$. Similarly, let the generative model have a marginal $p(\rvx)$, and $p(\rvz_t) = \int p(\rvz_t|\rvx) p(\rvx) d\rvx$. So far the assumptions are the same in our family. 

They assume that $q(\rvz_t|\rvx) = \mathcal{N}(\rvz_t, \rvx, t)$, which corresponds to a variance exploding (VE) diffusion process, with $t = \sigma_\lambda^2$, so $\lambda = -\log(t)$. 
Importantly, they make the assumption that $p(\rvz_t|\rvx) = q(\rvz_t|\rvx)$, i.e. that the forward process for $p$ equals the forward process for $q$.
Given these assumptions, \cite{lyu2012interpretation} show that:
\begin{align}
\frac{d}{dt} D_{KL}(q(\rvz_t)||p(\rvz_t)) 
&= - \frac{1}{2}
 D_{F}(q(\rvz_t)||p(\rvz_t))
\end{align}
Which, given the noise schedule $\lambda = -\log(t)$, can be rewritten as:
\begin{align}
\frac{d}{d\lambda} D_{KL}(q(\rvz_t)||p(\rvz_t)) 
&= \frac{1}{2} \sigma_\lambda^2
 D_{F}(q(\rvz_t)||p(\rvz_t))
\label{eq:lyu}
\end{align}
which looks a lot like our \Eqref{eq:df2}. One difference are that in \Eqref{eq:df2}, the left-hand-side distributions are joint distributions, and $q$ conditions on $\rvx$, while \Eqref{eq:lyu} is about the unconditional $q$. Another key difference is that for \Eqref{eq:df2} we need fewer assumptions: most importantly, we do not make the assumption that $p(\rvz_t|\rvx) = q(\rvz_t|\rvx)$, since this assumption does \emph{not} hold for the family of diffusion models we consider. Before or during optimization, $p(\rvz_t|\rvx)$ might be very far from $q(\rvz_t|\rvx)$. After optimization, $p(\rvz_t|\rvx)$ might be close to $q(\rvz_t|\rvx)$, but we still can't assume they're equal. In addition, we're mostly interested in the properties of the loss function during optimization, since that's when we're using our loss for optimization. We for this reason, our Theorem \ref{theorem:df} is a lot more relevant for optimization.

\section{Implementation details}
Instead of uniformly sampling $t$, we applied the low-discrepency sampler of time that was proposed by \citet{kingma2021variational}, which has been shown to effectively reduce the variance of diffusion loss estimator and lead to faster optimization. The model is optimized by \emph{Adam}~\citep{kingma2014adam} with the default hyperparameter settings. We clipped the learning gradient with a global norm of 1. 

For the adaptive noise schedules, we divided the range of $[\lambda_{\min}, \lambda_{\max}]$ into $100$ evenly spaced bins. During training, we maintaiedn an exponential moving average of $w(\lambda) ||\bepsilon - \hat{\bepsilon}_{\bT}(\rvz_\lambda;\lambda)||_2^2 $ with a decay rate 0.999, and a constant initialization value of 1 for each bin.

Below we elaborate the implementation details specific for each task.
\paragraph{ImageNet 64x64.} For class-conditional generation on ImageNet 64x64, we applied the ADM U-Net architecture from \citet{dhariwal2021diffusion}, with dropout rate $0.1$. We didn't use any data augmentation. The model was trained with learning rate $1e-4$, exponential moving average of 50 million images and learning rate warmup of 10 million images, which mainly follows the configuration of \citet{karras2022elucidating}. We employed 128 TPU-v4 chips with a batch size of 4096 (32 per chip). We trained the model for 700k iterations and reported the performance of the checkpoint giving the best FID score (checkpoints were saved and evaluated on every 20k iterations). It took around 3 days for a single training run. For training noise schedule and sampling noise schedule of DDPM sampler, we set $\lambda_{\min} = -20$ and $\lambda_{\max} = 20$. We fixed the noise schedule used in sampling to the cosine schedule for the DDPM sampler, and the EDM (sampling) schedule for the EDM sampler (see Table \ref{table:noise_schedules} for the formulations). We adopted the same hyperparameters of EDM sampler from \citet{karras2022elucidating} with no changes (i.e., Table 5 in their work, column `ImageNet-Our model'). Both DDPM and EDM samplers took 256 sampling steps. 

\paragraph{ImageNet 128x128.} For class-conditional generation on ImageNet 128x128, we heavily followed the setting of \emph{simple diffusion}~\citep{hoogeboom2023simple}. Specifically, we used their `U-ViT, L' architecture, and followed their learning rate and EMA schedules. The data was augmented with random horizontal flip. The model is trained using 128 TPU-v4 chips with a batch size of 2048 (16 per chip). We trained the model for 700 iterations and evaluated the FID and inception scores every 100k iterations. The results were reported with the checkpoint giving the best FID score. It took around 7 days for a single run. We set $\lambda_{\min} = -15 + s$ and $\lambda_{\max} = 15 + s$, where $s = \log(64/d)$ is the shift of the weighting function, with 64 being the base resolution and $d$ being the model resolution ($d=128$ for this task). The DDPM sampler used for evaluation used `shifted-cosine' noise schedule (Table \ref{table:noise_schedules}) and took 512 sampling steps. 

\section{Relationship with low-bit training}

Various earlier work, such as \citep{kingma2018glow}, found that maximum likelihood training on 5-bit data can lead to perceptually higher visual quality than training on 8-bit data (at the cost of a decrease in color fidelity). A likely reason that this leads to improved visual quality, is that this allows the model the spend more capacity on modeling the bits that are most relevant for human perception.

In \citep{kingma2018glow}, training on 5-bit data was performed by adding uniform noise to the data, before feeding it to the model. It was found that adding Gaussian noise had a similar effect as uniform noise. As we have seen, in the case of diffusion models, adding Gaussian noise is equivalent to using a the weighted objective with a monotonic weighting function.

Therefore, training on 5-bit data is similar to training using a monotonic weighting function in case of diffusion models. We can wonder: which weighting function emulates training on 5-bit data? Here, we'll attempt to answer this question.

\subsection{The shape of $\frac{d}{d\lambda} \Lt{\lambda}$ for low-bit data}
Note that the results of Appendix \ref{app:time_derivative_of_dlt} can also be written as:
\begin{align}
\frac{d}{d\lambda} \Lt{\lambda}
=
\frac{1}{2} 
\E_{\bepsilon \sim \mathcal{N}(0, \mathbf{I})}\left[
||\bepsilon - \hat{\bepsilon}_{\bT}(\rvz_\lambda;\lambda)||_2^2
\right]
\end{align}
This allows us to rewrite the weighted loss as simply:
\begin{align}
\mathcal{L}_w(\rvx) 
&=
\int_{\lmax}^{\lmin}
\frac{d}{d\lambda} \Lt{\lambda}
\, w(\lambda)
\;d\lambda
\end{align}
To understand the effect of $\lambda$, we can plot the $\frac{d}{d\lambda} \Lt{\lambda}$ as a function of $\lambda$. 

We'd like to plot $\frac{d}{d\lambda} \Lt{\lambda}$ dof different choices of bit precision. This will tell us where the different bits 'live' as a function of $\lambda$. Since training different diffusion models on different bit precisions is very expensive, we instead use an approximation. In particular, we assume that the data $\rvx$ is univariate, with a uniform distribution $q(\rvx)$ over the $2^n$ possible pixel values, where $n$ is the bit precision. The data is $\rvx$ is, as usual, normalized to $[-1,1]$. We then let the model $p(\rvz_\lambda)$, for each choice of $\lambda$, be the optimal model: $p(\rvz_\lambda) := \int q(\rvx) q(\rvz_\lambda|\rvx) d\rvz$, which is a univariate mixture-of-Gaussians, where each mixture component is a Gaussian centered one of the $2^n$ possible pixel values. In this case, $\Lt{\lambda} := \Ltfull{\lambda}$.

We plot the function $\E_{q(\rvx)}[\Lt{\lambda}]$, for differences choices of $n = 1, ..., 8$, below:

\includegraphics[width=1.0\textwidth]{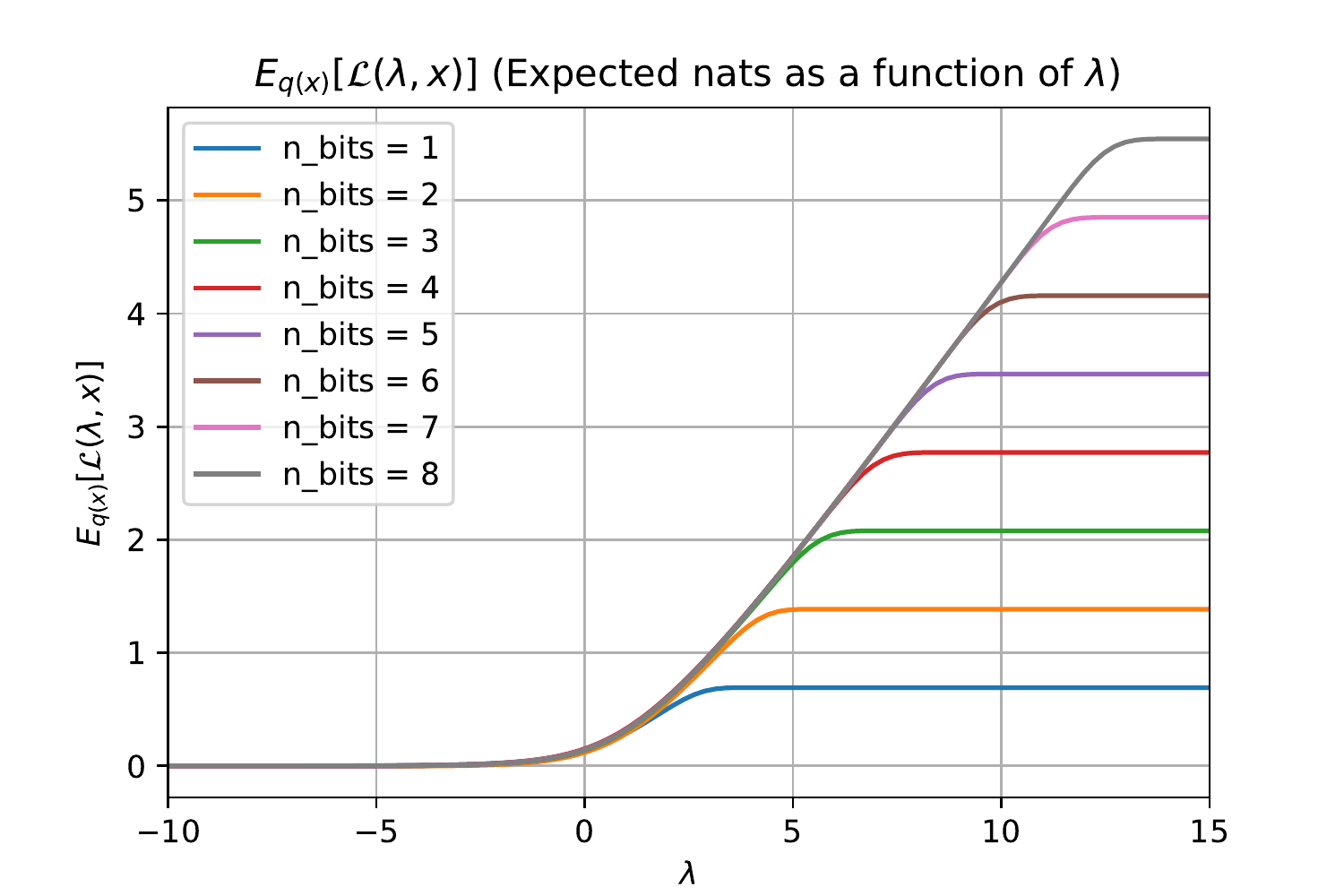}

Next, we plot $\E_{q(\rvx)}[\frac{d}{d\lambda} \Lt{t_\lambda}]$, for each $n$:

\includegraphics[width=1.0\textwidth]{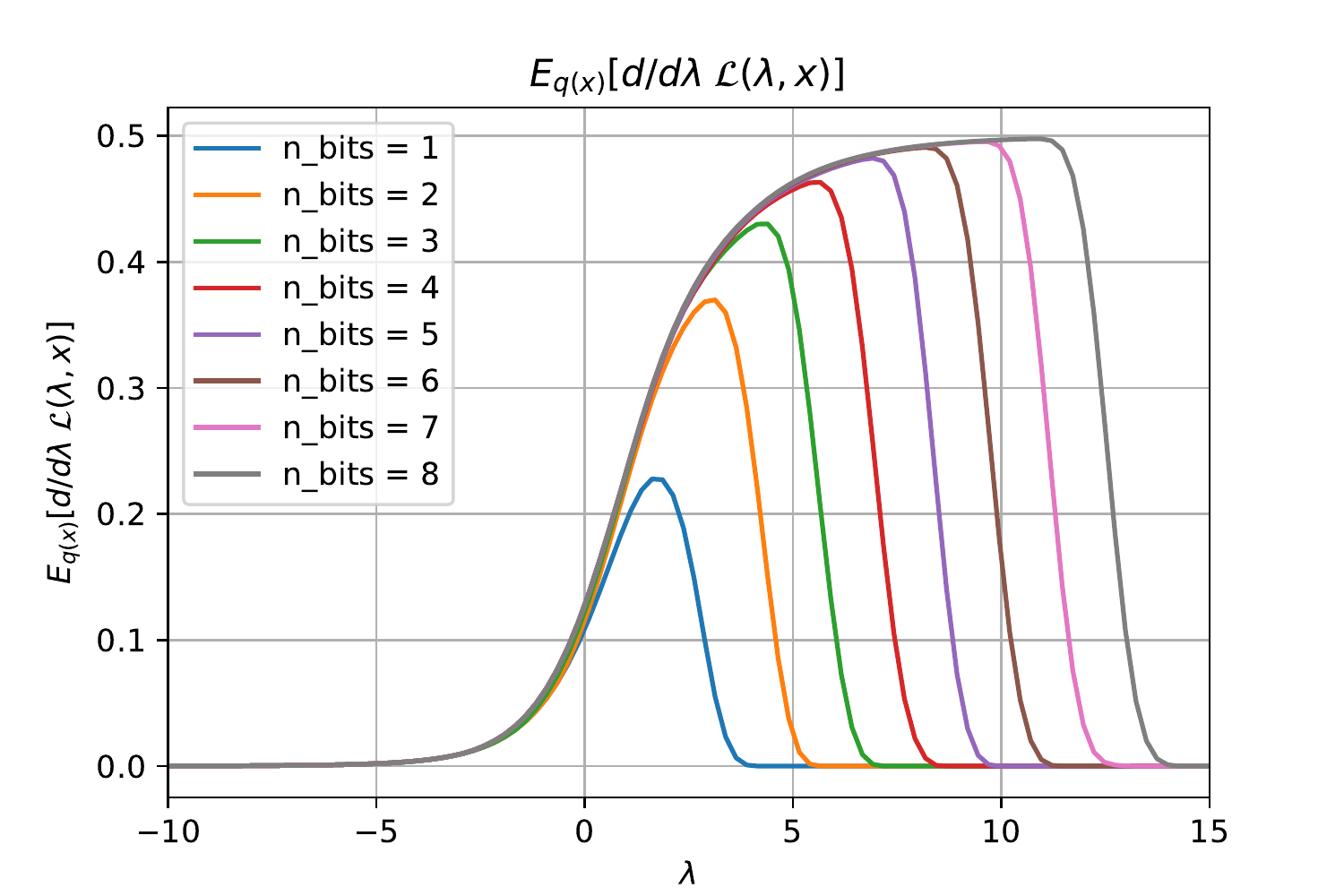}

We can visualize the contribution of each additional bit to the loss, by substracting the curve for $n-1$ bits from the curve for $n$ bits:

\includegraphics[width=1.0\textwidth]{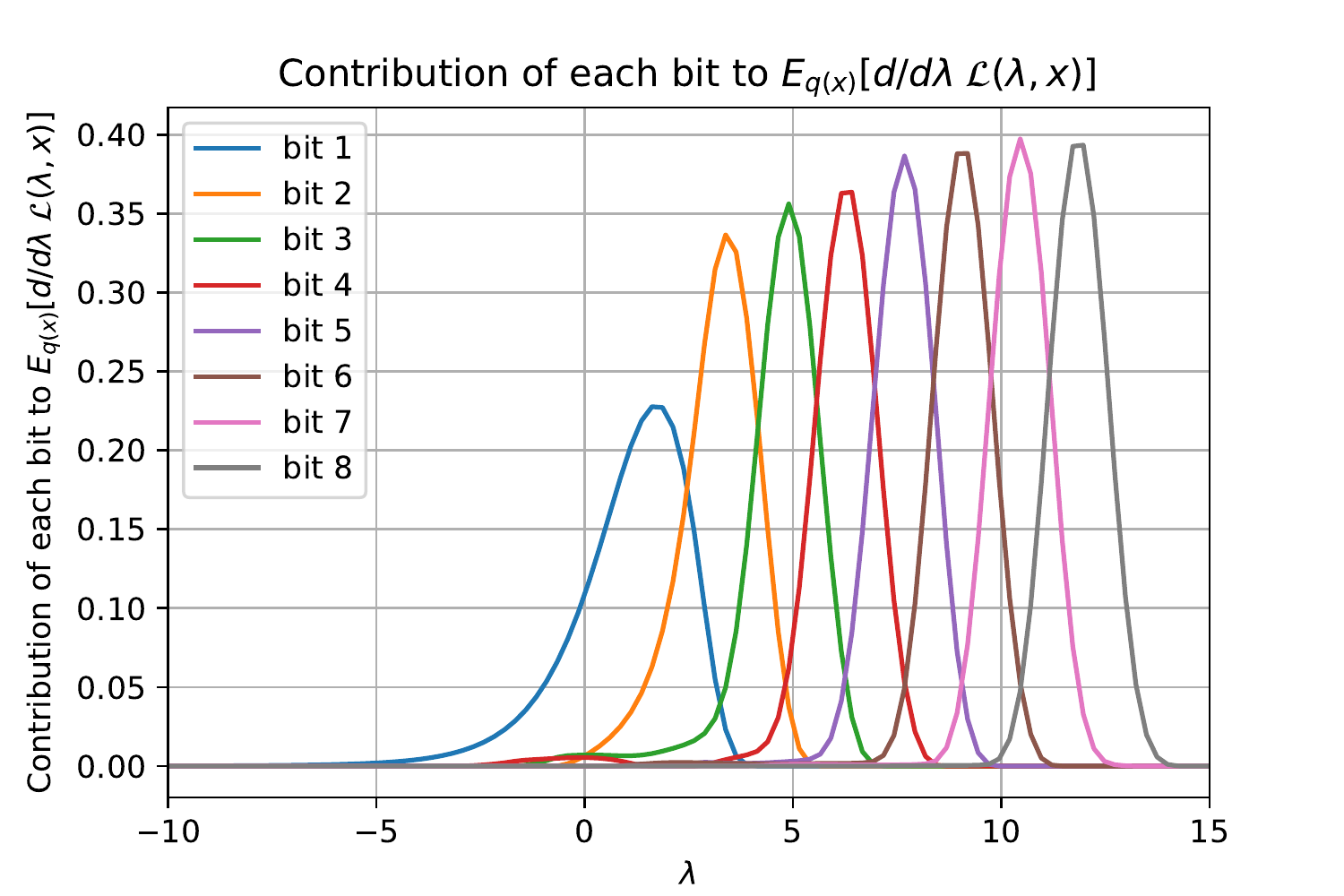}

The area under each of these curves is exactly 1 bit.

Training with 5-bit data is similar to training on the original (8-bit) data, but with a sigmoidal weighting that goes down sharply between $\lambda=7.5$ and $\lambda=10$. In fact, the 5-bit unweighted loss curve is very similar to the 8-bit loss curve, when using the following weighting function:
\begin{align}
    w(\lambda) = F_{\mathcal{N}}((-2(\lambda - 8.4)))
\label{eq:w_5bit}
\end{align}
where $F_{\mathcal{N}}$ is the CDF of a standard Gaussian, as visualized in this graph:

\includegraphics[width=1.0\textwidth]{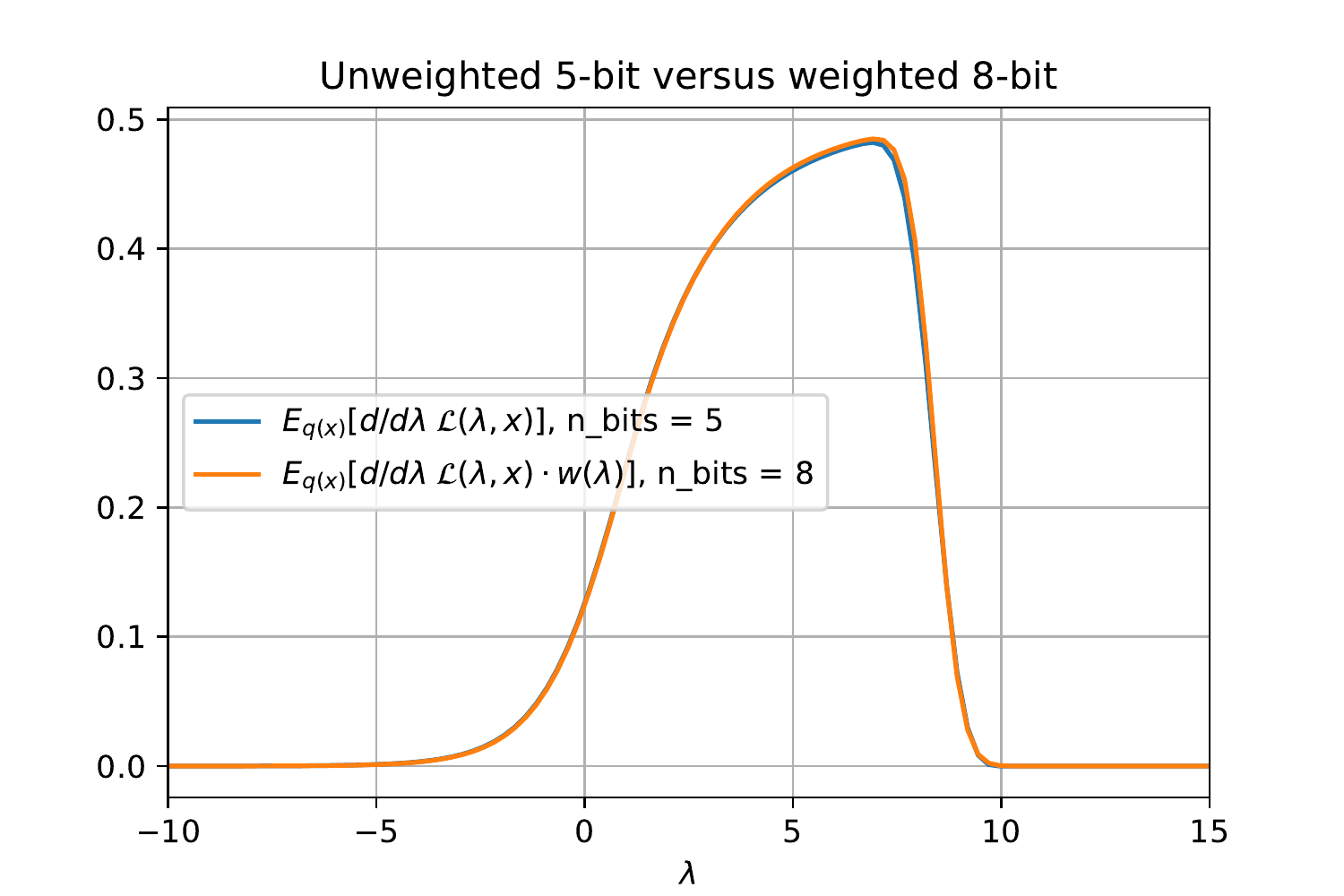}

Interestingly, the weighting function in \Eqref{eq:w_5bit} gives much more weight to low noise levels than the weighting functions used in this paper, as visualized below, where '5-bit-like weighting' is the weighting form \Eqref{eq:w_5bit}:

\includegraphics[width=1.0\textwidth]{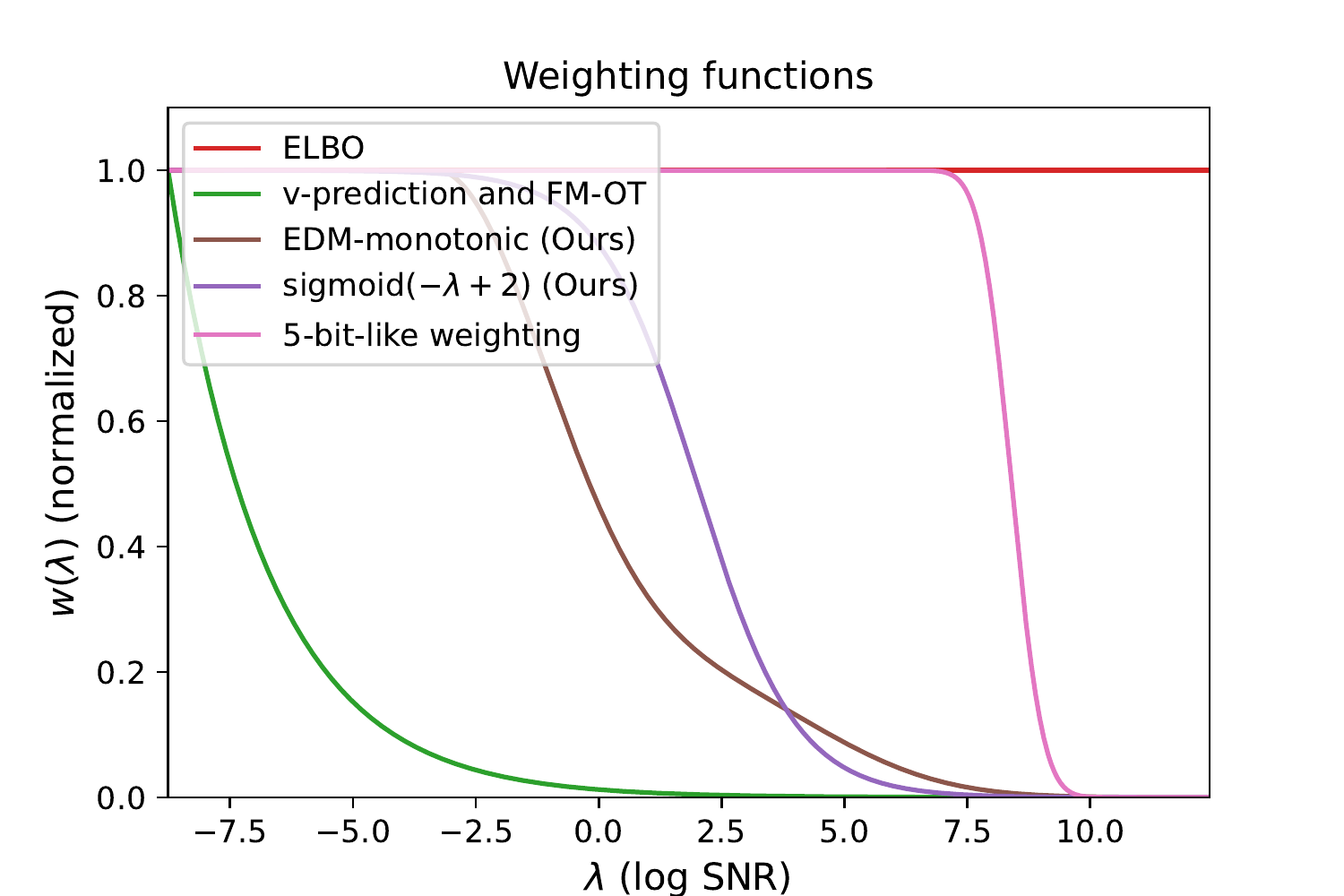}

\section{Limitations}

It is important to emphasize that our empirical results, like other deep learning approaches, depend on the choice of hyper-parameters. A change in, for example, the dataset or spatial resolution will generally require re-tuning of optimization hyperparameters, architectural choices and/or weighting functions. Such re-tuning can be time consuming and costly.

\section{Broader impact}

While our work primarily focuses on theoretical developments in the understanding and optimization of diffusion models, the advancements could have broader implications, some of which could potentially be negative. The development of more efficient and effective generative models could, on one hand, propel numerous beneficial applications, such as art and entertainment. However, it is also worth acknowledging the potential misuse of these technologies.

One notable concern is the generation of synthetic media content, for example to mislead. These fraudulent yet realistic-looking images and videos could be used for spreading disinformation or for other malicious purposes, such as identity theft or blackmail.

Regarding fairness considerations, generative models are typically trained on large datasets and could therefore inherit and reproduce any biases present in the training data. This could potentially result in unfair outcomes or perpetuate harmful stereotypes if these models are used in decision-making processes or content generation.

Mitigation strategies to address these concerns could include the gated release of models, where access to the model or its outputs is regulated to prevent misuse. Additionally, the provision of defenses, such as methods to detect AI generated media, could be included alongside the development of the generative models. Monitoring mechanisms could also be implemented to observe how a system is being used and to ensure that it learns from feedback over time in an ethical manner.

These issues should be considered when applying the techniques we propose. The community should strive for an open and ongoing discussion about the ethics of AI and the development of strategies to mitigate potential misuse of these powerful technologies.

\section{Samples from our model trained on 512 $\times$ 512 ImageNet}
\label{sec:samples}

Below we provide random samples from our highest-resolution (512x512) model trained on ImageNet. We did not cherry-pick, except that we removed depictions of humans due to ethical guidelines. Samples in Figures \ref{fig:samples1} and \ref{fig:samples2} are generated without guidance, while samples in Figures \ref{fig:samples3} and \ref{fig:samples4} are generated with guidance strength 4.

\begin{figure}
    \centering
    \includegraphics[width=\textwidth]{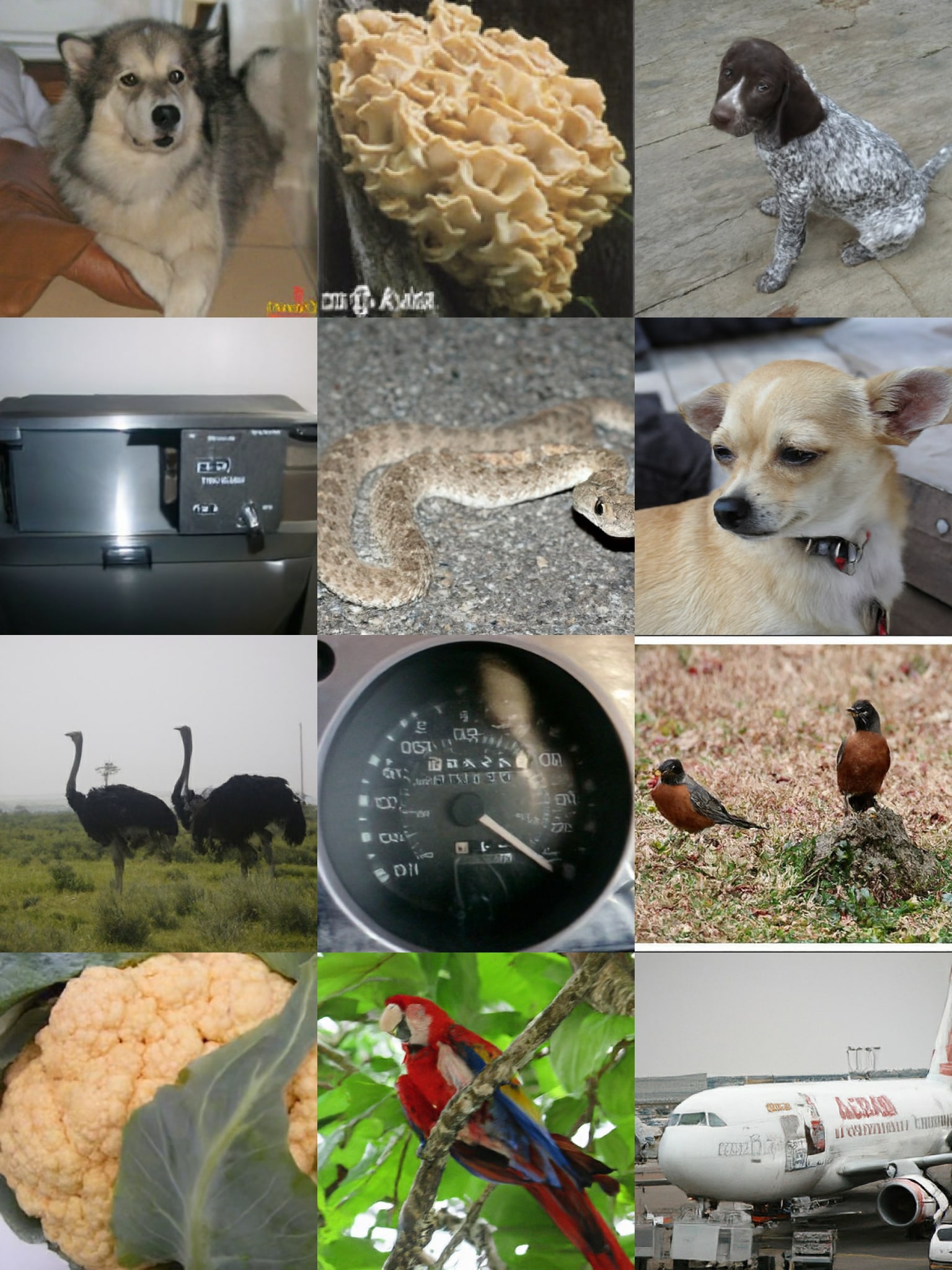}
    \caption{Random samples from our 512x512 ImageNet model, without guidance.}\label{fig:samples1}
\end{figure}

\begin{figure}
    \centering
    \includegraphics[width=\textwidth]{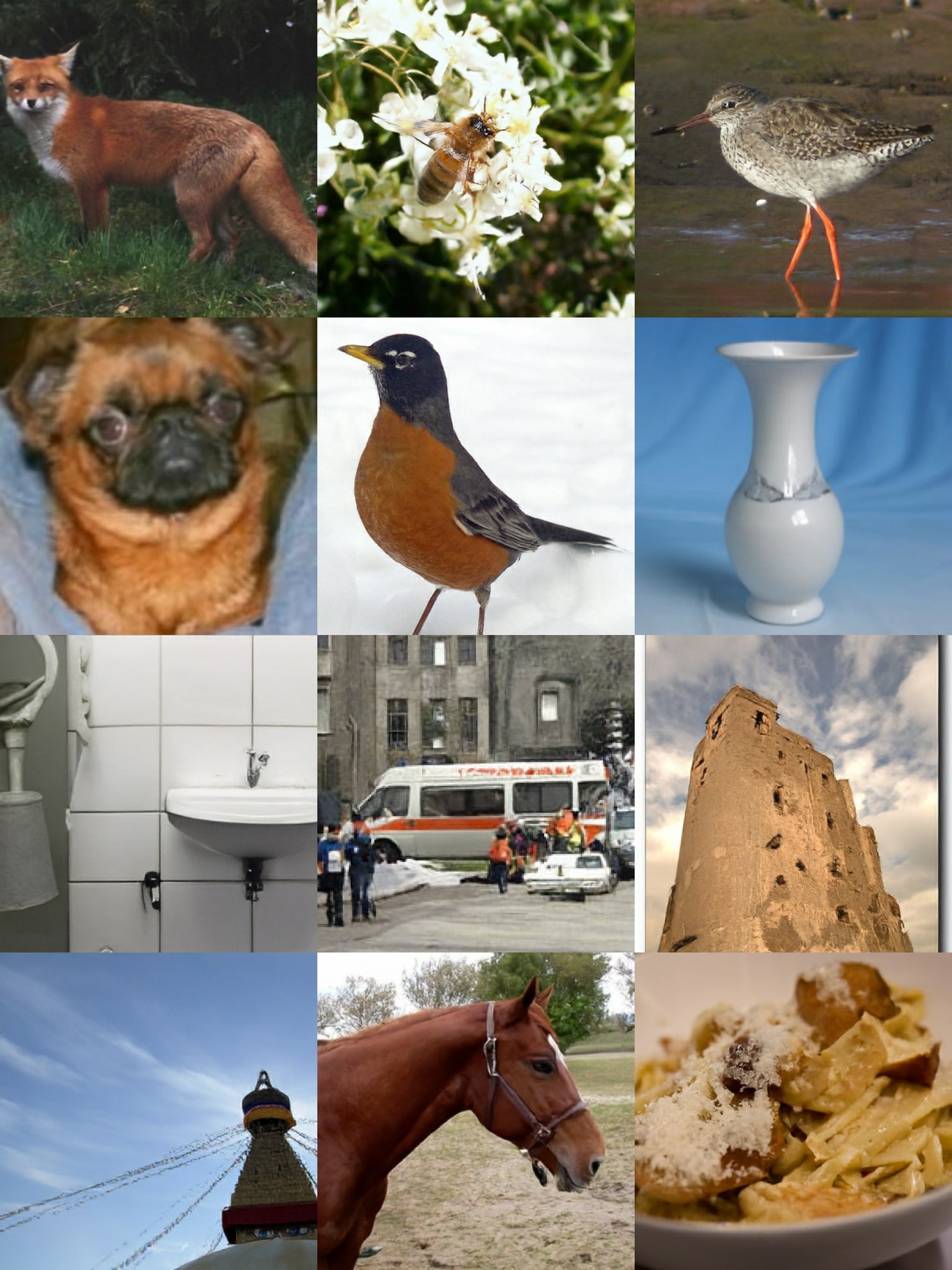}
    \caption{More random samples from our 512x512 ImageNet model, without guidance.}\label{fig:samples2}
\end{figure}

\begin{figure}
    \centering
    \includegraphics[width=\textwidth]{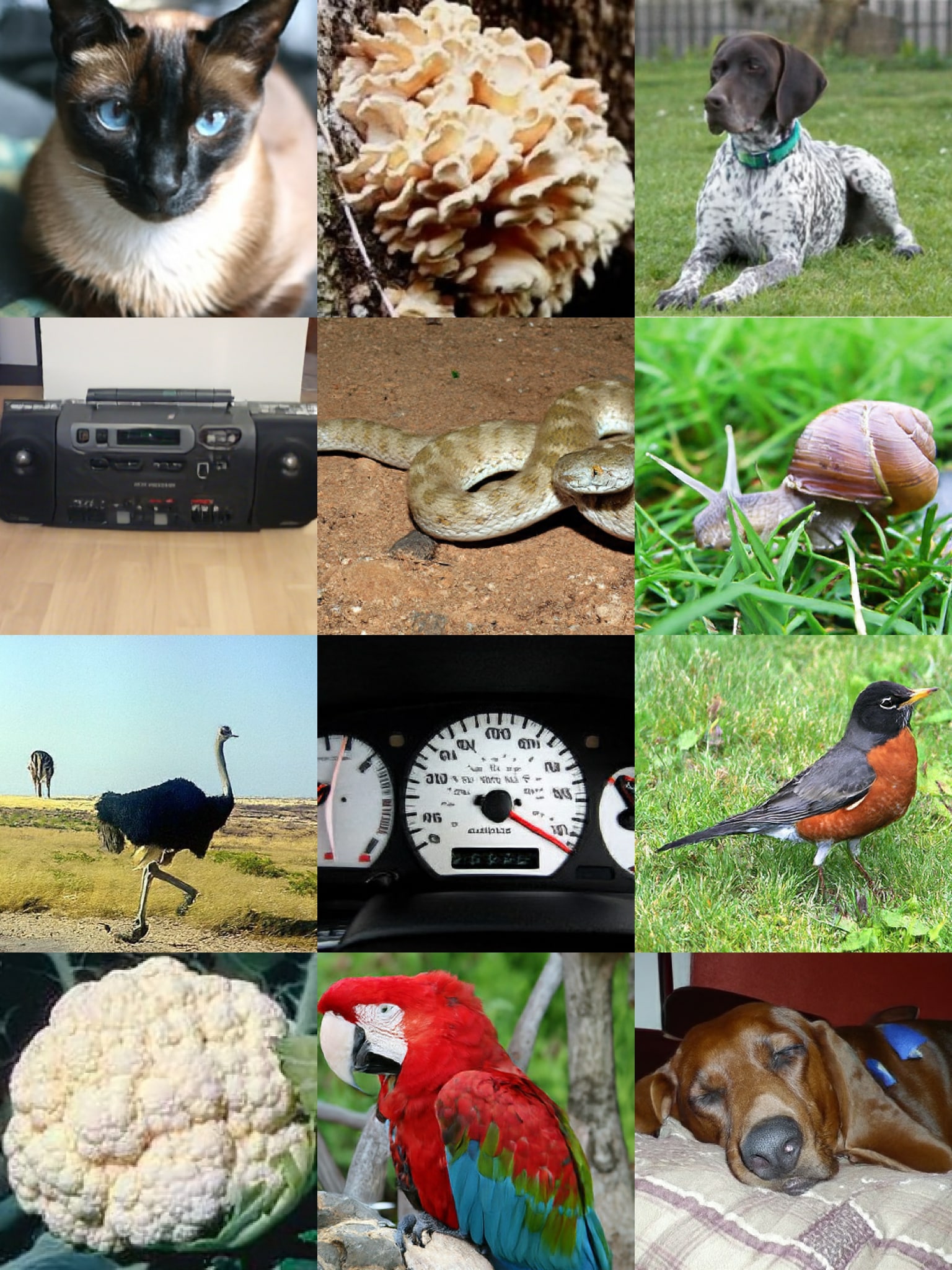}
    \caption{Random samples from our 512x512 ImageNet model, with guidance strength 4.}\label{fig:samples3}
\end{figure}

\begin{figure}
    \centering
    \includegraphics[width=\textwidth]{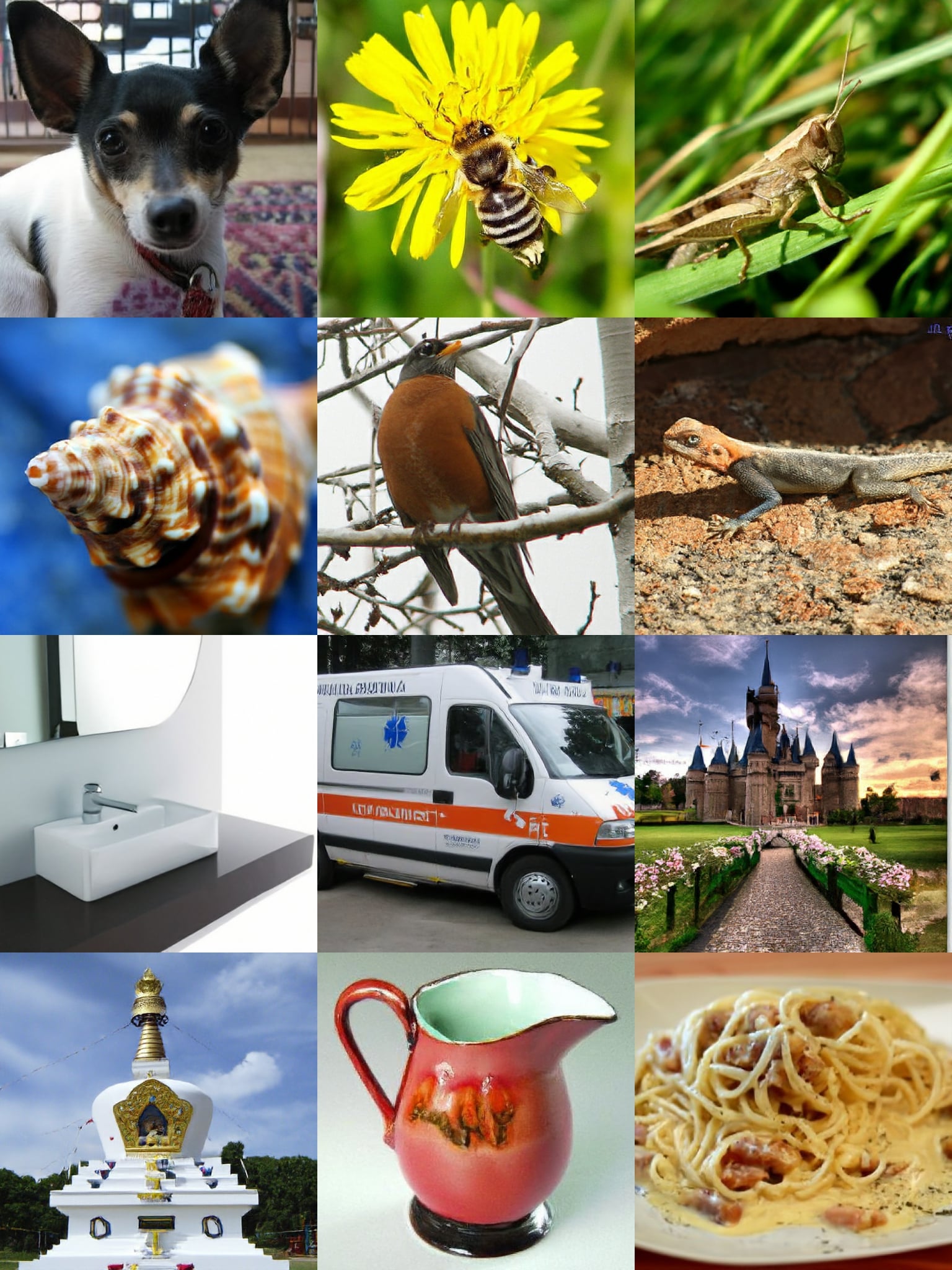}
    \caption{More random samples from our 512x512 ImageNet model, with guidance strength 4.}\label{fig:samples4}
\end{figure}

\end{document}